\documentclass{article}

\usepackage[utf8]{inputenc}
\usepackage{amssymb}
\usepackage{amsmath}
\usepackage{amsthm}
\usepackage{ dsfont }
\usepackage{bbm}
\usepackage{graphicx}
\usepackage{babel,blindtext}
\usepackage{comment}
\usepackage[colorinlistoftodos]{todonotes}
\usepackage{thm-restate}
\usepackage[colorlinks=true, citecolor=blue,linkcolor = blue, urlcolor=black]{hyperref}
\usepackage[letterpaper,top=2cm,bottom=2cm,left=3cm,right=3cm,marginparwidth=1.75cm]{geometry}
\usepackage{wrapfig}

\usepackage{natbib}
\setcitestyle{authoryear,open={(},close={)}}
\usepackage{graphicx}

\usepackage{caption}
\usepackage{subcaption}
\usepackage{float}

\newtheorem{lemma}{Lemma}
\newtheorem{fact}[lemma]{Fact}

\newtheorem{example}{Example}

\newcommand{\norm}[1]{\left\lVert#1\right\rVert}
\newcommand{\R}{\mathbb{R}}

\DeclareMathOperator*{\argmin}{arg\,min}

\title{How and When Random Feedback Works:\\ A Case Study of Low-Rank Matrix Factorization }

\author{
Shivam Garg \\ Stanford University \\ \href{mailto:shivamg@cs.stanford.edu}{shivamg@cs.stanford.edu}  \and Santosh S. Vempala \\ Georgia Tech \\ \href{mailto:vempala@gatech.edu} {vempala@gatech.edu}}

\date{}
\begin{document}
\maketitle
%

%





\begin{abstract}
  The success of gradient descent in ML and especially for learning neural networks is remarkable and robust. In the context of how the brain learns, one aspect of gradient descent that appears biologically difficult to realize (if not implausible) is that its updates rely on feedback from later layers to earlier layers through the same connections. Such bidirected links are relatively few in brain networks, and even when reciprocal connections exist, they may not be equi-weighted. Random Feedback Alignment \citep{lillicrap2016random}, where the backward weights are random and fixed, has been proposed as a bio-plausible alternative and found to be effective empirically. We investigate how and when feedback alignment (FA) works,  focusing on  one  of  the  most  basic  problems  with  layered  structure  —  low-rank matrix  factorization.   In  this  problem,  given  a  matrix $Y_{n\times m}$,  the  goal is to find a low rank factorization $Z_{n \times r}W_{r \times m}$ that minimizes the error $\|ZW-Y\|_F$.  Gradient descent solves this problem optimally.  We show that  FA  finds  the  optimal  solution  when $r\ge \mbox{rank}(Y)$.    We also shed light on {\em how} FA works.  It is observed empirically that the forward weight matrices and (random) feedback matrices come closer during FA updates.  Our analysis rigorously derives this phenomenon and shows how it facilitates convergence of FA*, a closely related variant of FA. We also show that FA can be far from optimal when $r < \mbox{rank}(Y)$.  This is the first provable separation result between gradient descent and FA. Moreover, the representations found by gradient descent and FA can be almost orthogonal even when their error $\|ZW-Y\|_F$ is approximately equal. As a corollary, these results also hold for training two-layer linear neural networks when the training input is isotropic, and the output is a linear function of the input.

\end{abstract}

\newpage

\section{Introduction}

Information Processing in the brain is hierarchical, with multiple layers of neurons from perception to cognition, and learning is believed to be largely based on updates to synaptic weights. 
These weight updates depend on error information that may only be available in the downstream (higher-level) areas. An algorithmic challenge faced by the brain is the following: how to update the weights of earlier layers using the error information from later layers, despite local structural constraints? For example, in the visual cortex, the weight update to earlier layers --- which detect low-level information such as edges in an image --- may depend on higher-level information in the image that is available only after downstream processing. 

\begin{wrapfigure}{r}{0.45\textwidth}
\vspace{-16pt}
\centering
\includegraphics[width=0.45\textwidth]{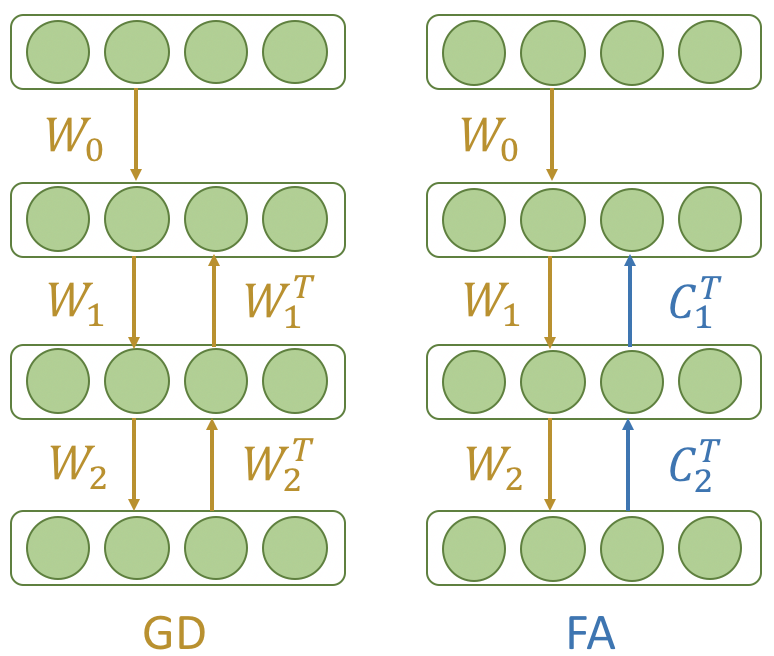}
  \captionof{figure}{\small{Gradient descent uses transpose of the forward weights for backward feedback while feedback alignment replaces them by fixed random weights.
  }}
    \label{fig:gd_fa_cartoon}
\end{wrapfigure}

In artificial neural networks, gradient descent via backpropagation \citep{rumelhart1986learning} has been a very successful method of making weight updates. However, it is unclear whether gradient descent is  biologically plausible due to its non-local updates \citep{crick1989recent}. In particular, the update to earlier layers involves feedback from later layers through backward weights that are transposed copies of the corresponding forward weights (see Fig. \ref{fig:gd_fa_cartoon}). This requires equi-weighted bidirectional links between neurons, which are rare in the brain. This issue was first identified by \citet{grossberg1987competitive}, who called it the weight transport problem.

Rather surprisingly, \citet{lillicrap2016random} found that neural networks are able to learn even when the backward feedback weights are random and fixed, independent of the forward weights. This biologically plausible variant of gradient descent is known as {\em Feedback Alignment (FA)}. Feedback alignment and its variants \citep{nokland2016direct} have been shown to be effective for many problems ranging from language modeling to neural view synthesis \citep{launay2020direct}. At the same time, they do not match the performance of gradient descent for large-scale visual recognition problems \citep{bartunov2018assessing, moskovitz2018feedback} such as ImageNet \citep{russakovsky2015imagenet}.

These observations raise many questions: How and when does random feedback work? Is there any fundamental sense in which feedback alignment is inferior to gradient descent? How different are the representations found using feedback alignment and gradient descent? 
Alongside the biological motivation, these questions are also important for getting a better understanding of the landscape of possible optimization algorithms.

\paragraph{Problem formulation and contributions.}
In this paper, we investigate these questions by considering one of the most basic problems with layered structure --- low-rank matrix factorization \citep{du2018algorithmic, valavi2020revisiting, ye2021global}. In  this  problem,  given  a  matrix $Y_{n\times m}$,  the  goal is to find a low rank factorization $Z_{n \times r}W_{r \times m}$ that minimizes the error 
\begin{align}
\label{eq:error_fa}
\|Z_{n \times r}W_{r \times m}-Y_{n\times m}\|_F^2.    
\end{align}

The gradient flow (GD) update (gradient descent with infinitesimally small step size) for this problem is given by
\begin{align}
\label{eq:gd_update}
\begin{split}
    \frac{dZ}{dt} &= (Y - \hat{Y})W^T\\
    \frac{dW}{dt} &= Z^T(Y - \hat{Y}),
\end{split}
\end{align}
where $\hat{Y} = ZW$.

From prior work \citep[Theorem 39]{bah2019learning}, we know that gradient flow starting from randomly initialized $Z$ and $W$ converges to the optimal solution almost surely. The layered structure and optimality of gradient flow makes low-rank matrix factorization an ideal candidate   for understanding the performance of feedback alignment.

The feedback alignment (FA) update is given by
\begin{align}
\label{eq:fa_update}
\begin{split}
    \frac{dZ}{dt} &= (Y - \hat{Y})C^T\\
    \frac{dW}{dt} &= Z^T(Y - \hat{Y})
\end{split}
\end{align}
Note that the only difference from gradient flow update is that the backward feedback weight $W^T$ is replaced by $C^T$  in the expression for $\frac{dZ}{dt}$ . Here, $C$ is some (possibly random) fixed matrix.

Empirically, it is observed that the backward feedback weights ($C$, in this case) and the forward weights ($W$) come closer during feedback alignment updates \citep{lillicrap2016random}. After the forward and backward weights are sufficiently aligned, the feedback alignment update is similar to the gradient flow update. This alignment between the forward weights and the backward feedback weights led to the name feedback alignment, and is considered to be the main reason for the effectiveness of this algorithm.

However, the phenomenon of alignment has turned out to be hard to establish rigorously. One reason behind this is that the alignment between forward and backward weights may not increase monotonically (see Example \ref{ex:convergence-1d} for details). We observe that a small tweak to the feedback alignment update where $W$ is updated optimally, leads to monotonically increasing alignment between $C$ and $W$ (for an appropriately defined notion of alignment). We call this version of feedback alignment FA*, and its updates are given by
\begin{align}
\label{eq:fa_plus_update}
\begin{split}
    \frac{dZ}{dt} &= (Y - \hat{Y})C^T\\
    W &= (Z^TZ)^{-1}Z^TY.
\end{split}
\end{align}
Notice that the only difference between FA and FA* is that $W$ is chosen optimally (given $Z$ and $Y$) in the FA* update, while it moves in the negative gradient direction in the FA update. The update to $Z$ remains the same, and involves a fixed feedback matrix $C$.
\begin{figure*}[t]
    \centering 
\begin{subfigure}{0.49\textwidth}
  \includegraphics[width=\linewidth]{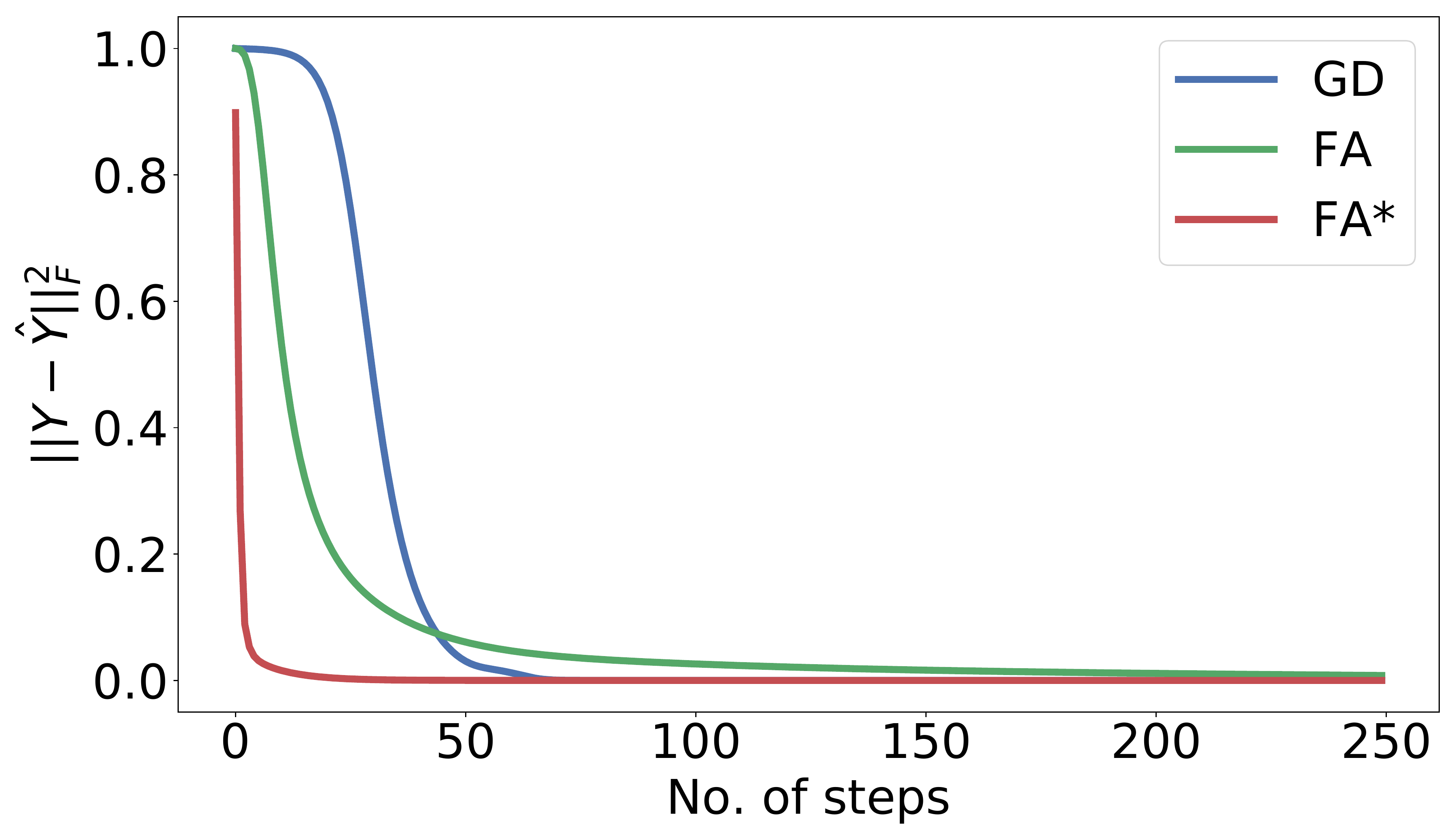}
  \caption{}
  \label{fig:separation_GD_FA_1}
\end{subfigure}
\begin{subfigure}{0.49\textwidth}
  \includegraphics[width=\linewidth]{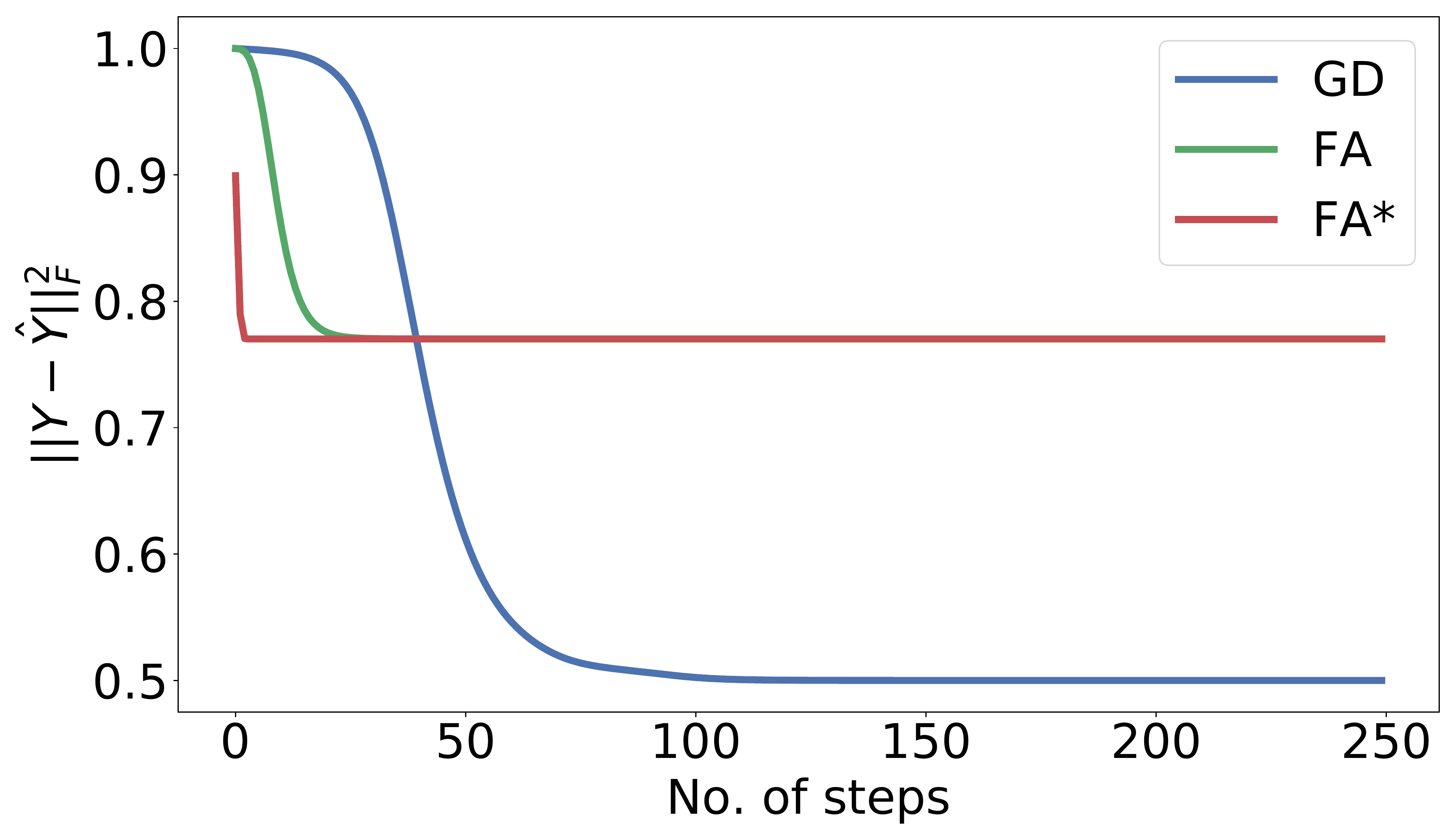}
  \caption{}
  \label{fig:separation_GD_FA_2}
\end{subfigure}
\caption{(a) Feedback alignment (both FA (\ref{eq:fa_update}) and FA* (\ref{eq:fa_plus_update})) converge to the optimal solution when $r \geq rank(Y)$. In this plot, $n = m = 500$ and $r = rank(Y) = 50$. (b) Feedback alignment solution can be far from optimal when $r < rank(Y)$. In this plot, $n = m = rank(Y) = 500$ and $r = 50$. Gradient descent finds the optimal solution in both the cases. \vspace{-10pt}}
\label{fig:separation_GD_FA}
\end{figure*}
We show that FA* initialized with an arbitrary full column rank $Z$ converges to a stationary point where $(Y - \hat{Y})C^T = 0$ (Theorem \ref{thm:convergence-mat-fac}). Our analysis rigorously demonstrates the phenomenon of alignment, and sheds light on how it facilitates convergence (Section \ref{sec:convergence_fa}).

Convergence of FA to a stationary point has also been shown in past works \citep{baldi2018learning, lillicrap2016random}. \citet{baldi2018learning} proves convergence of feedback alignment for learning one-hidden layer neural networks with linear activation, starting from arbitrary initialization. For low-rank matrix factorization, this implies convergence of FA to a stationary point. However, as we discussed, alignment between forward and backward weights may not increase monotonically in FA. Due to this, these works are not able to say much about the dynamics of alignment. The main feature of our analysis is that by analyzing a slight variant of FA (FA*), we obtain a better understanding of the phenomenon of alignment and its implications for convergence.

After analysing how feedback alignment works, we shift our attention to the question of when it works. 
We characterize the solution $\hat{Y}$ at the stationary points of feedback alignment, when $C$ is chosen randomly (Lemma \ref{lem:FA_sol_char}). Building on this characterization, we show that feedback alignment finds the optimal solution when $r \geq rank(Y)$ (Theorem \ref{thm:FA_over_param_opt}). However, it can be far from optimal when $r < rank(Y)$ (Theorem \ref{thm:FA_separation_GD}). To the best of our knowledge, this is the first provable separation result between gradient flow and feedback alignment (see Fig. \ref{fig:separation_GD_FA} for an illustration).  

Moreover, the representations found by feedback alignment and gradient flow are very different. We show that even when their errors $\norm{ZW-Y}_F^2$ are approximately equal, the representations found ($Z$) can be almost orthogonal (Theorem \ref{thm:FA_GD_rep_orthogonal}).

Since the stationary point equations for FA and FA* are same, these results about suboptimality of feedback alignment and difference in representations apply to both versions of feedback alignment.

In summary, we give a comprehensive analysis of how and when feedback alignment works, focusing on the problem of low-rank matrix factorization. Here is a list of our contributions:
\begin{enumerate}
    \item We prove convergence of feedback alignment (FA*) to a stationary point, shedding light on the dynamics of alignment and its implications for convergence (Section \ref{sec:convergence_fa}).
    \item We show that feedback alignment (both FA and FA*) find the optimal solution when $r \geq rank(Y)$, but can be far from optimal when $r < rank(Y)$. This shows provable separation between feedback alignment and gradient flow (Section \ref{sec:stationary_points}). 
    \item We characterize the representations found by feedback alignment (both FA and FA*), and show that they can be very different from the representations found by gradient flow, even when their errors are approximately equal (Section \ref{sec:stationary_points}).
\end{enumerate}

As a corollary, all our results also hold for training  two-layer linear neural networks, assuming the training input is isotropic and the output is a linear function of the input (Section \ref{sec:lin_nn}).  We defer all proofs and simulation details to the appendix. 

\paragraph{Notation. }
For any matrix $M$, $M(t)$ denotes its value at time $t$. We will not explicitly show $t$ when it is clear from context. $\sigma_i(M)$ denotes the $i^{\text{th}}$ largest singular value of $M$. $M^{(i)}$ denotes the $i^{\text{th}}$ column of $M$. $\norm{M}_F$ denotes the Frobenius norm of $M$ and $\norm{v}$ denotes the $\ell_2$ norm of vector $v$. 

\section{Related Work}
\paragraph{Feedback alignment.} \cite{lillicrap2016random} show convergence of feedback alignment dynamics for learning one-hidden layer neural networks with linear activation, starting from zero initialization. \citet{baldi2018learning} generalize this result to arbitrary initialization, and also show convergence for linear neural networks of arbitrary depth when the input and all hidden layers are one dimensional. 

In recent work, \citep{song2021convergence} study feedback alignment for highly overparameterized one-hidden layer neural networks where the width of the hidden layer is much larger than the size of training set. This work builds on past work on Neural Tangent Kernels \citep{jacot2018neural}, and shows that feedback alignment converges to a solution with zero training error.  Contrary to the popular understanding of feedback alignment, they show that forward and backward weights may not align in this highly overparameterized regime. However, in the parameter regime typically encountered in practice, alignment is a robust phenomenon \citep{lillicrap2016random}.

\citet{refinetti2021align} obtain a set of ODEs that describe the progression of feedback alignment test error for neural networks in certain parameter regimes. Using simulations, and analysis of these ODEs at initialization, they argue that neural network training proceeds in two phases: the initial alignment phase where the forward and backward weights align with each other, followed by a memorization phase where learning happens. 
In Section \ref{sec:convergence_fa}, we show that while such a progression can take places in simple cases (see Example \ref{ex:convergence-1d}), in general, the dynamics are much more involved with highly interleaved phases. This paper also presents intuition about the behaviour of feedback alignment for deeper networks, and possible reasons for its poor performance with  convolutional neural networks (CNNs).

The focus of our work is twofold: (i) understanding how feedback alignment works by studying the dynamics of alignment and its impact on convergence, (ii) understanding when feedback alignment works by contrasting its solution and representations with gradient descent. Our work complements the existing line of work on understanding feedback alignment.

\paragraph{Biologically plausible learning.}
Many algorithms have been proposed to address the weight transport problem \citep{lillicrap2020backpropagation}. 
Most of these algorithms either encourage alignment between forward and backward weights implicitly \citep{lillicrap2016random, nokland2016direct, moskovitz2018feedback, akrout2019deep}, or  learn weights that try to preserve information between adjacent layers \citep{bengio2014auto, lee2015difference, kunin2019loss, kunin2020two}. A parallel line of work studies how training algorithms can be implemented in the brain using spiking neurons without distinct inference (forward propagation) and training (backward propagation) phases \citep{xie2003equivalence, bengio2017stdp, scellier2017equilibrium,  whittington2017approximation, guerguiev2017towards, sacramento2018dendritic}. More recent work more directly models plasticity and inhibition in the brain and shows that memorization and learning are emergent phenomena \citep{papadimitriou2020brain, dabagia2021assemblies}.

Building the mathematical foundation of such biologically plausible algorithms can lead to illuminating insights applicable to the brain as well as to the general theory of optimization. Our work can be viewed as progress in this direction.

\section{Convergence}
\label{sec:convergence_fa}
In this section, we show that FA* (\ref{eq:fa_plus_update}) converges to a stationary point satisfying $(Y - \hat{Y})C^T = 0$, where $\hat{Y} = ZW$ and $W = (Z^TZ)^{-1}Z^TY$. 

\begin{restatable}{theorem}{thmconvg}
\label{thm:convergence-mat-fac} Let $Z(0)$ be full column rank. 
For any $\epsilon > 0$ and 
\[
T \geq  \frac{24}{\epsilon} \left( \frac{\sigma_1\left(Y\right)  \sigma_1\left(C\right)  \sigma_1\left(Z(0)\right)^6  \sqrt{r \ min(m, n)}}{\sigma_r\left(Z(0)\right)^5} \right),
\]
FA* dynamics (\ref{eq:fa_plus_update}) satisfy 
\[
\min_{t \leq T} \ \norm{(Y - \hat{Y}(t))C^T}_F^2 \leq \epsilon. 
\]
Moreover,
\[
\lim_{t \to \infty} \norm{(Y - \hat{Y}(t))C^T}_F^2 = 0.
\]
\end{restatable}
Note that the time for convergence of minimum of $ \norm{(Y - \hat{Y}(t))C^T}_F^2$ depends linearly on $\frac{1}{\epsilon}$. We describe the complete proof of Theorem \ref{thm:convergence-mat-fac} in Appendix \ref{sec:convg_proof}.

To understand this result, we first discuss a toy example where $m = 1$ (recall $Y$ is an $n \times m$ matrix).

\begin{example}
\label{ex:convergence-1d}
Suppose we want to factorize $y_{n \times 1}$ as $\hat{y}_{n \times 1} = Z_{n \times r} w_{r \times 1}$, and we use $c_{r \times 1}$ for feedback. FA* update is given by
\begin{align*}
    \frac{dZ}{dt} &= (y - \hat{y})c^T\\
    w &= (Z^TZ)^{-1}Z^Ty.
\end{align*}
This gives
\begin{align*}
    \frac{d \ \norm{(y-\hat{y})c^T}_F^2}{dt} &= -2 \norm{y - \hat{y}}^2 \norm{c}^2 c^Tw\\
    \frac{d \ c^T w}{dt} &= c^T (Z^TZ)^{-1} c \ \norm{y - \hat{y}}^2.
\end{align*}
Note that $\frac{d \ c^T w}{dt} \geq 0$. And $\frac{d \ \norm{(y-\hat{y})c^T}_F^2}{dt} \geq 0$ only when $c^T w < 0$. So $c^T w$  increases with time (when $\norm{y - \hat{y}}^2 > 0$). $\norm{y - \hat{y}}^2 \norm{c}^2$ increases in the beginning if $c^T w < 0$, but it starts decreasing once $c^T w > 0$, and eventually goes to $0$ (see Fig. \ref{fig:dynamics_1}). This shows how alignment of $c$ and $w$ (measured by $c^T w$) facilitates convergence.

The alignment between $c$ and $w$  increases monotonically in FA* dynamics. This is not true in FA dynamics. FA update is given by 
\begin{align*}
    \frac{dZ}{dt} &= (y - \hat{y})c^T\\
    \frac{dw}{dt} &= Z^T(y - \hat{y}).
\end{align*}
This gives
\begin{align*}
    \frac{d \ c^T w}{dt} = c^T Z^T(y - \hat{y}).
\end{align*}
Suppose $Z(0) = -y c^T$ and $w(0) = 0$. In this case, at $t=0$, $\frac{d \ c^T w}{dt} = - \norm{c}^2 \norm{y}^2 < 0$. Therefore, alignment between $c$ and $w$ can also decrease in FA dynamics. This makes FA* more suitable to understand the dynamics of alignment and its implications for convergence.
\end{example}
\begin{figure*}[t]
    \centering 
\begin{subfigure}{0.5\textwidth}
  \includegraphics[width=  \linewidth]{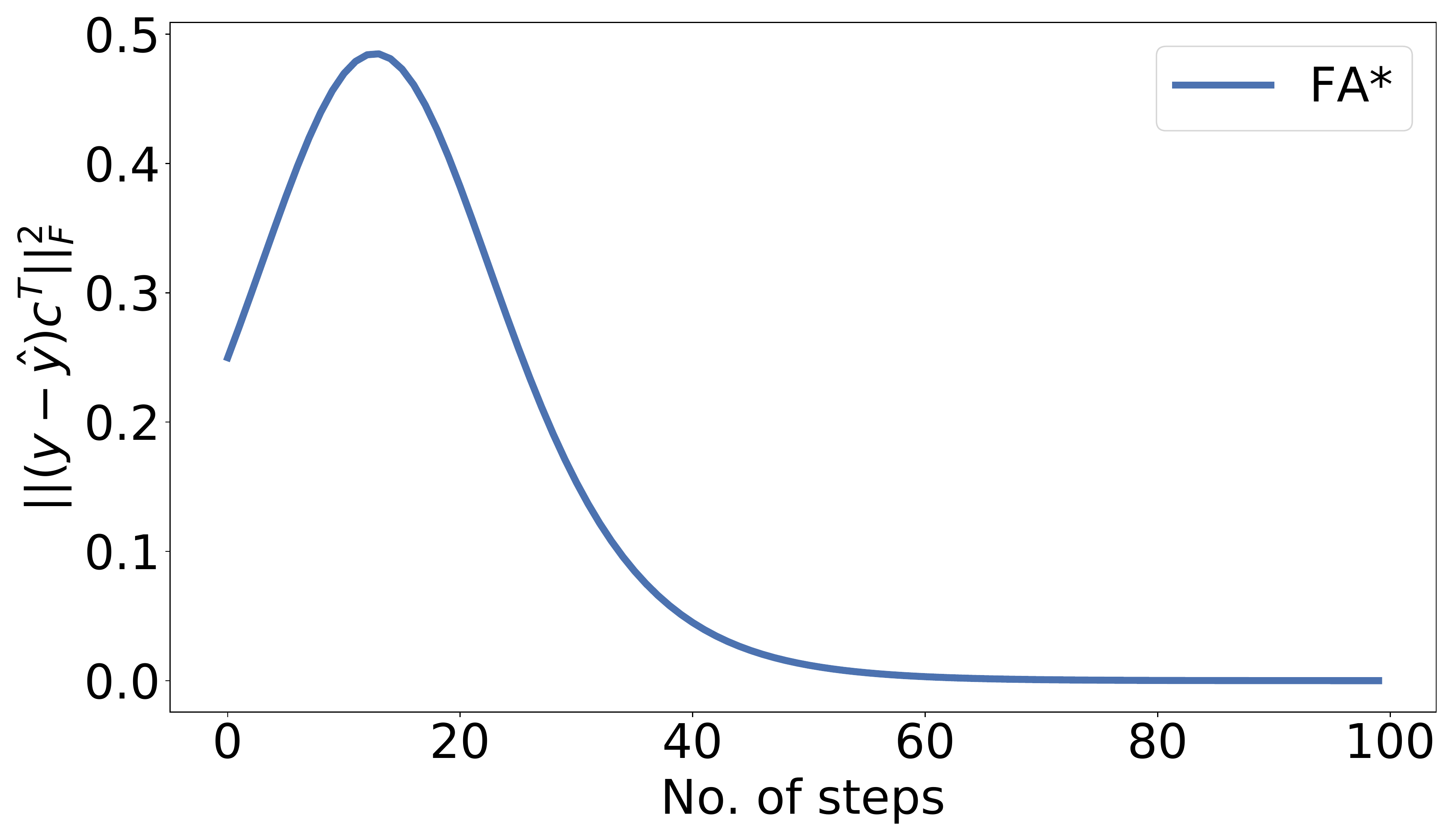}
  \caption{}
  \label{fig:dynamics_1}
\end{subfigure}
\begin{subfigure}{0.5\textwidth}
  \includegraphics[width=\linewidth]{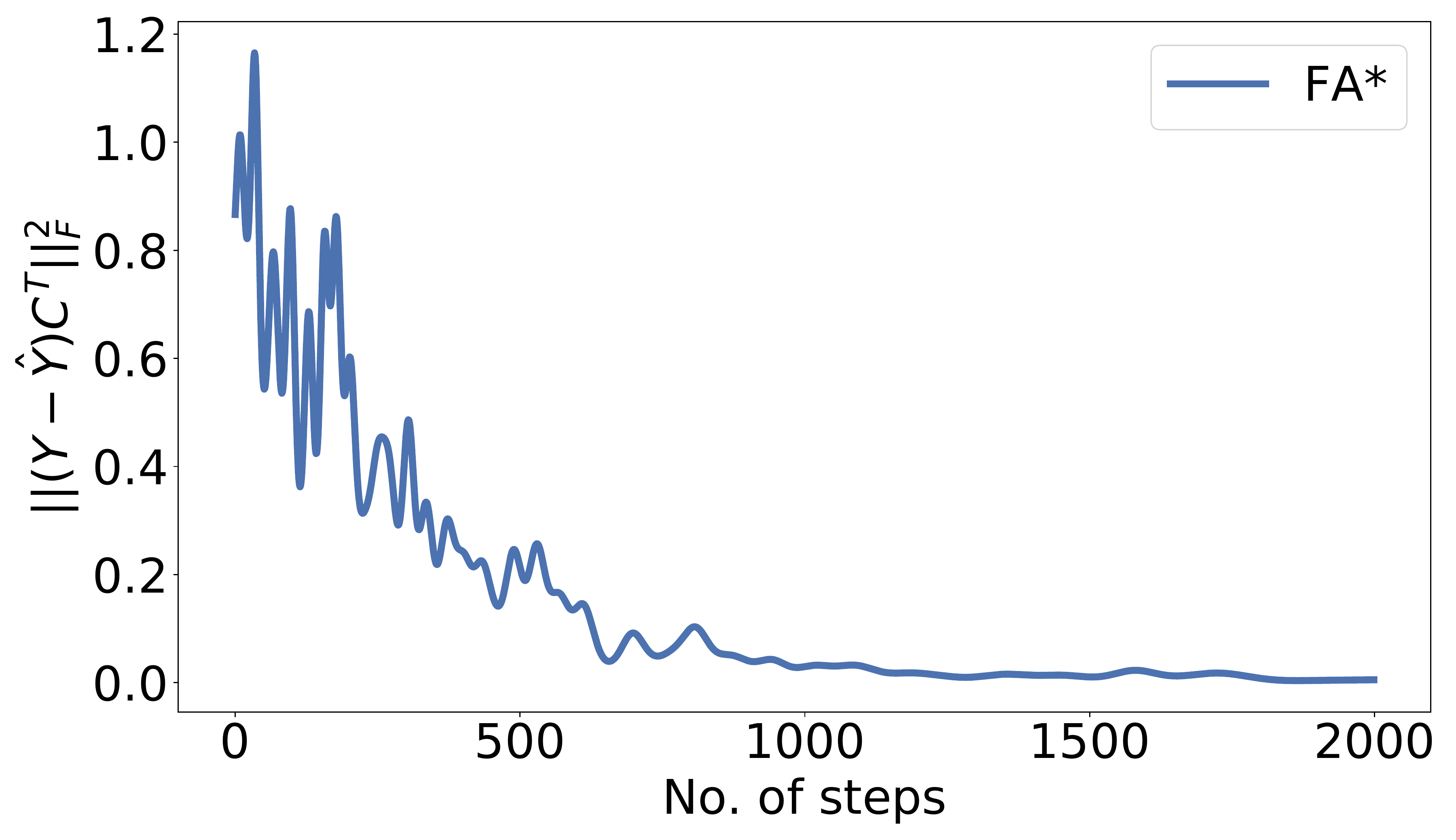}
  \caption{}
  \label{fig:dynamics_2}
\end{subfigure}
\caption{(a)When $y$ and $\hat{y}$ are column vectors, $\norm{(y-\hat{y})c^T}_F^2$ increases monotonically in the beginning followed by monotonic decrease (see Example \ref{ex:convergence-1d}). (b) For general matrices $Y$ and $\hat{Y}$, progression of $\norm{(Y-\hat{Y})C^T}_F^2$ can be highly non-monotonic. \vspace{-10pt} }
\label{fig:dynamics}
\end{figure*}

In this example with $m = 1$, we saw that the loss $\norm{(Y - \hat{Y})C^T}_F^2$ has an initial phase in which it increases monotonically, followed by a phase in which it decreases monotonically. However, the loss can be highly non-monotone in the general case. We illustrate this in Fig. \ref{fig:dynamics_2}, where we show the loss progression for FA* for the case where $m = n = 100$ and $r = 99$. In our simulations, we observe such highly non-monotonic behaviour when $r$ is close to $n$. We observe a similar highly non-monotone behavior of loss for FA as well. 

Therefore, we need a more careful analysis to understand the dynamics for the general case.
From the FA* update equations (\ref{eq:fa_plus_update}), we get $\frac{d \ Z^T Z}{dt} = 0$. That is, $Z^TZ$ does not change with time. For this discussion, let us assume $Z$ is initialized such that $Z(0)^TZ(0) = I$, which implies $Z^TZ = I$ throughout. 

Also, let $R$ denote the residual matrix $(Y - \hat{Y})C^T$, $A$ denote the alignment matrix $CW^T + WC^T$, and $\ell$ denote the loss $\norm{R}_F^2$. $R_i$ denotes the $i^\text{th}$ row of $R$ (viewed as a column vector). Using basic matrix calculus, we get
\begin{align}
    \frac{d \ell}{dt} &= -Tr(RAR^T) = -\sum_{i=1}^n R_i^T A R_i, \label{eq:FA_plus_loss_update} \\
    \frac{dA}{dt} &= 2 R^TR. \label{eq:FA_plus_A_update}
\end{align}
 Equation \ref{eq:FA_plus_loss_update} says that if $A$ is positive semi-definite (PSD), then the loss  $\ell$ decreases with time. Equation \ref{eq:FA_plus_A_update} says that $A$ becomes more PSD with time, that is, $x^T A x$ never decreases for any fixed $x$ (see Fig. \ref{fig:alignment_1}).  This is the sense in which alignment between $C$ and $W$ increases monotonically.  However, unlike Example \ref{ex:convergence-1d}, this is not sufficient to claim that loss starts decreasing monotonically after some time. This is because
 $A$ may never become PSD as there can be some $x$ for which $\frac{d \ x^T A x}{dt}$ remains $0$  after some time. We demonstrate this in Fig. \ref{fig:alignment_2} where we show an instance where the minimum eigenvalue of $A$ is monotonically increasing, but stays negative. That is, $A$ does not become PSD. However, $\ell$ still converges to zero (Fig. \ref{fig:dynamics_2} shows the corresponding loss progression).

\begin{figure*}[t]
    \centering 
\begin{subfigure}{0.5\textwidth}
  \includegraphics[width=\linewidth]{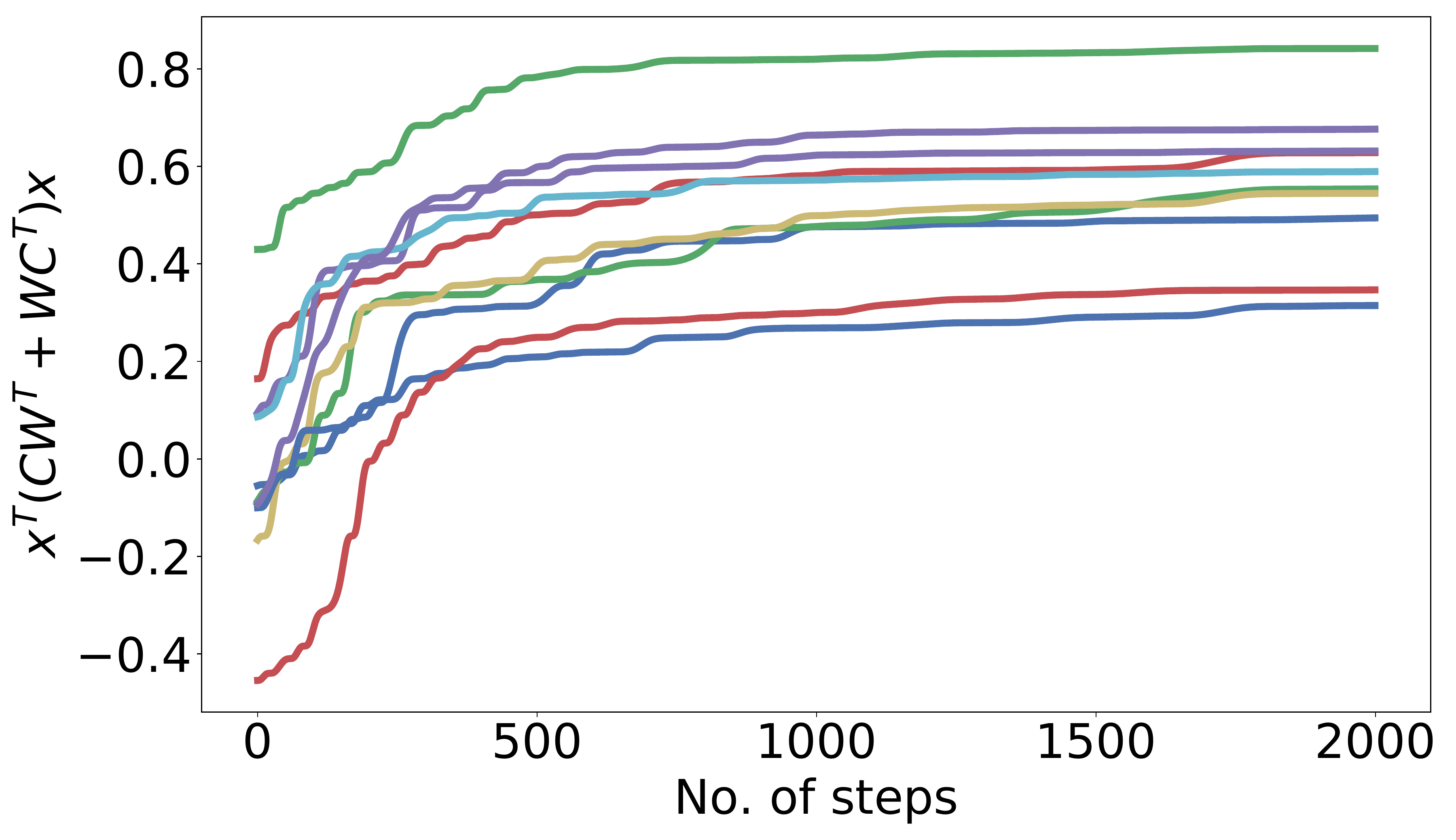}
  \caption{}
  \label{fig:alignment_1}
\end{subfigure}
\begin{subfigure}{0.5\textwidth}
  \includegraphics[width=\linewidth]{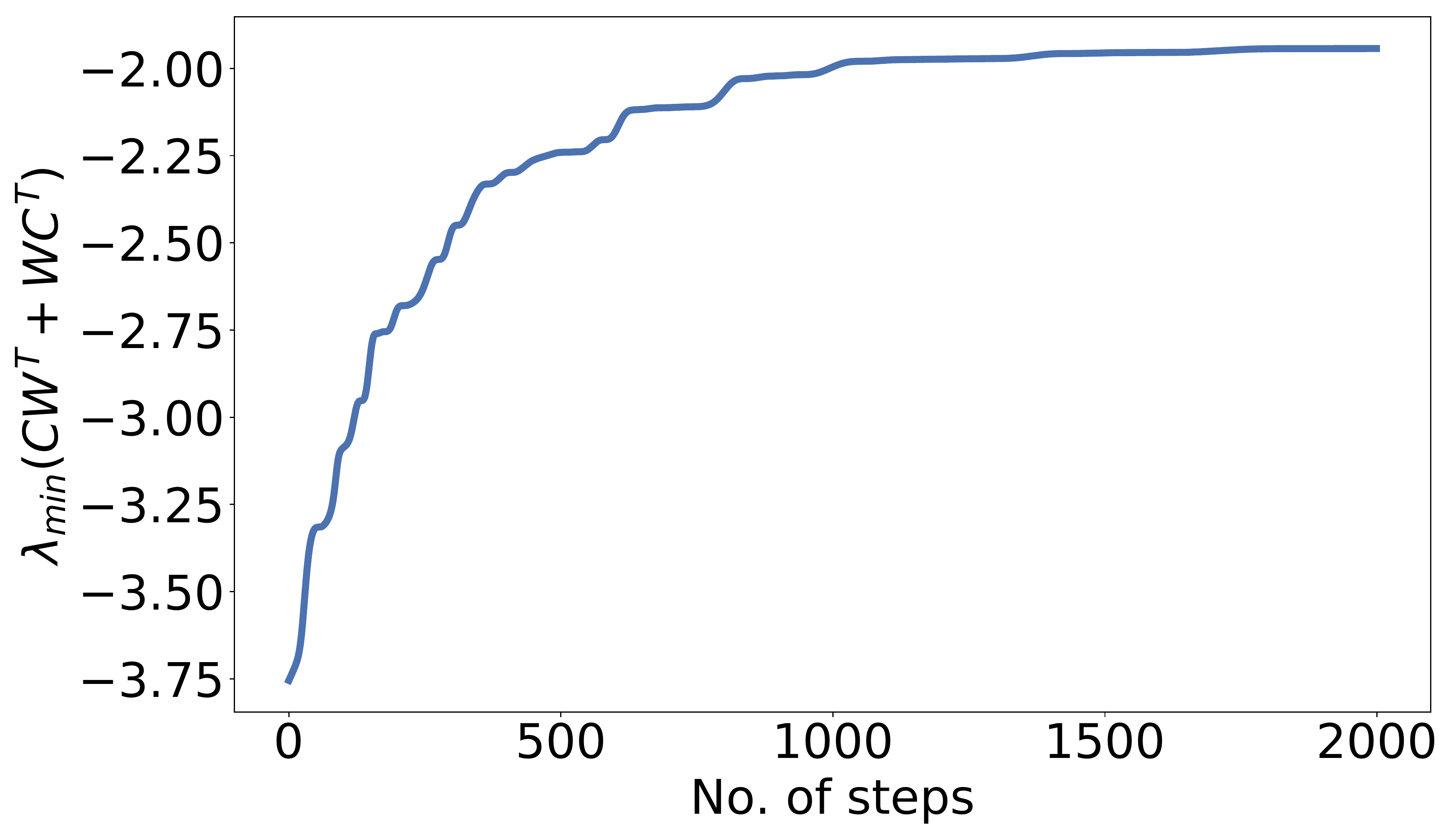}
  \caption{}
  \label{fig:alignment_2}
\end{subfigure}
\caption{ (a) $x^T (CW^T + WC^T)x$ vs time for 10 randomly chosen $x$. $x^T (CW^T + WC^T)x$ is monotonically increasing for all $x$. (b) Minimum eigenvalue of $CW^T + WC^T$ is monotonically increasing but can stay negative.\vspace{-10pt}}
\label{fig:alignment}
\end{figure*}
To get past this hurdle, we need to understand the directions $x$ for which $\frac{d \ x^T A x}{dt} > 0$.  Observe that  when $\frac{d\ell}{dt} > 0$, there is some row $R_j$  of $R$ such that $R_j^T A R_j < 0$. And $x^T A x$ increases sufficiently for all $x$ for which $R_j^T x$ is large enough. In other words, when the loss increases, $C$ and $W$ become better aligned with respect to the direction  which led to increase in loss ($R_j$), and all directions close to it.  And such an $R_j$ --- with respect to which $C$ and $W$ are not aligned, satisfying $R_j^TAR_j < 0$ --- must exist whenever the loss increases.  Therefore, the loss can not increase indefinitely. Using this idea, we bound the total possible increase in loss, $\int_0^T \frac{d\ell}{dt} \mathds{1} \left[\frac{d\ell}{dt} \geq 0 \right]$ dt,  for all $T$. Here, $\mathds{1}[\cdot]$ denotes the indicator function which is equal to $1$ if the condition inside the bracket is true, and $0$ otherwise.

Using a similar argument, we can bound the total time for which the loss is large and is either increasing or decreasing very slowly. That is, we bound $\int_0^T \mathds{1} \left[\ell \geq \epsilon \text{ and } \frac{d\ell}{dt} \geq -\delta  \right] dt$ for all $\epsilon > 0, \delta > 0, T \geq 0$. At any other time, if the loss is large, it has to decrease sharply.

Therefore, the loss cannot increase too much, and cannot be in a slowly decreasing phase for too long. 
Using this, we show a bound on the time by which loss goes below $\epsilon$ (first part of Theorem \ref{thm:convergence-mat-fac}). Building on these ideas, we can also show that the loss converges to $0$ eventually. We refer the readers to Appendix \ref{sec:convg_proof} for more details.

In summary, here is the crux of the argument: whenever a bad event happens (increase in loss or slow decrease in loss), the alignment between $C$ and $W$ increases with respect to the direction which caused the bad event ($R_j$), and all directions close to it. Such a bad event can not happen when $C$ and $W$ are sufficiently aligned with respect to all rows of $R$. Therefore a bad event can not happen many times. This identifies the directions with respect to which alignment increases, and how this phenomenon facilitates convergence.

\paragraph{Implications for FA.}

\begin{wrapfigure}{r}{0.47\textwidth}
\vspace{-16pt}
\centering
\includegraphics[width=0.47\textwidth]{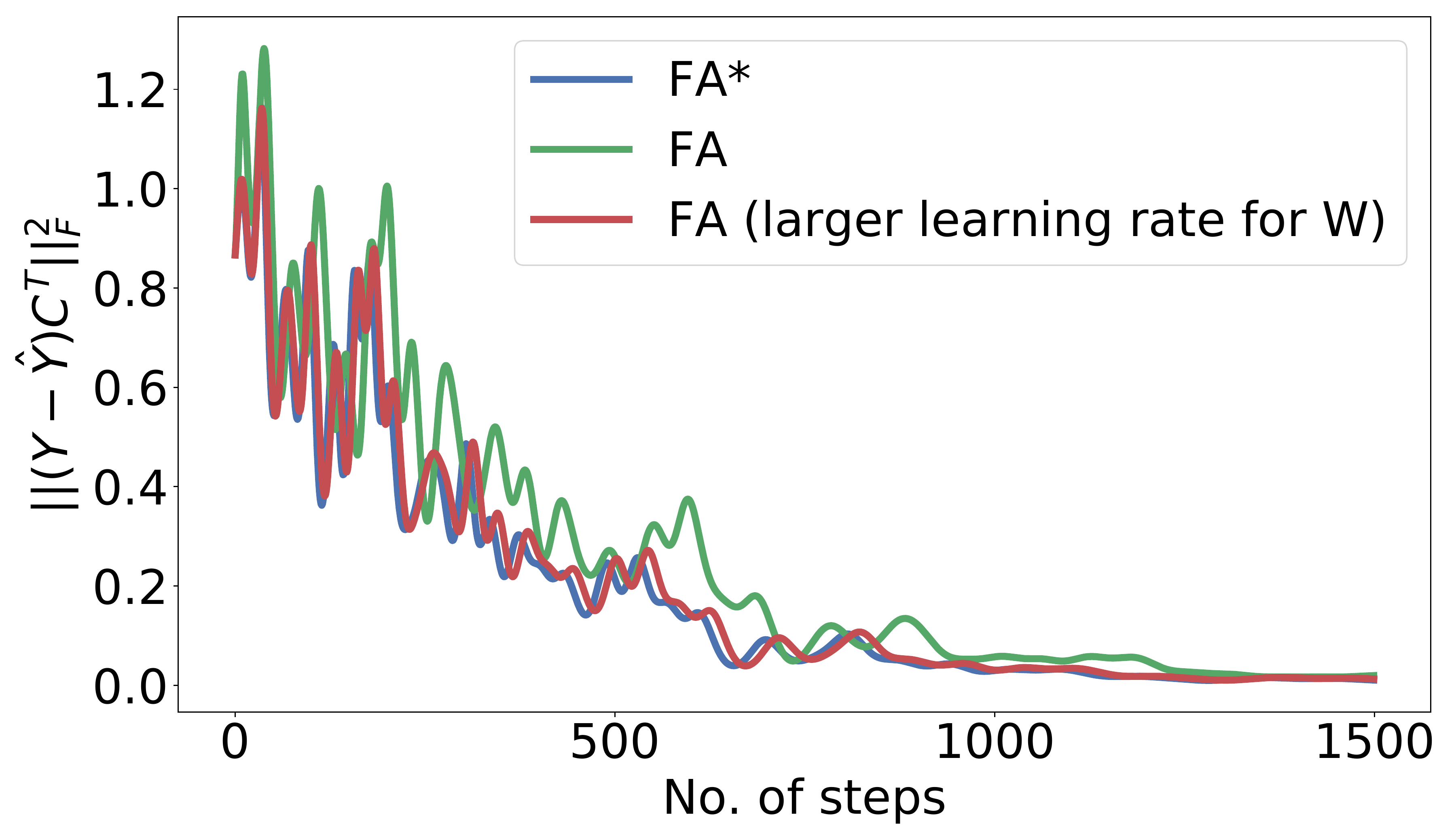}
  \captionof{figure}{\small{
  FA and FA* loss progression when $W$ is initialized optimally for FA.
  }}
    \label{fig:dynamics_comparison}
\end{wrapfigure}

 The only difference between FA and FA* is that we set $W$ optimally in the FA* update whereas we take the gradient step for $W$ in the FA update. As we discussed in Example \ref{ex:convergence-1d}, the dynamics for FA and FA* can generally be very different. However, if we initialize FA with the optimal $W$ (for the given $Z$) at $t=0$, we observe that its loss progression is similar to FA* (see Fig. \ref{fig:dynamics_comparison}).
The similarity is even more apparent if we initialize $W$ optimally and choose a larger learning rate for $W$ (compared to $Z$) in which case $W$ continues to be close to optimal throughout the dynamics. For reference, we also include a plot for FA with randomly initialized $W$ in Appendix \ref{sec:sim_details}.

More generally, we believe that the ideas behind understanding alignment for FA*  may also be helpful for FA. To see this, observe that for FA,   
\begin{align*}
\frac{dA}{dt} &= C(Y - \hat{Y})^TZ + Z^T(Y - \hat{Y})C^T,\\
    \frac{d^2 A}{ dt^2} &= 2 R^TR \\ &  - Z^T(Y-\hat{Y})C^TWC^T  - CWC^T(Y - \hat{Y})^TZ \\ &  -  Z^TZZ^T(Y - \hat{Y})C^T  - C(Y - \hat{Y})^TZZ^TZ.
\end{align*}
 Here, $A = CW^T + WC^T$ and $R = (Y - \hat{Y})C^T$. If $W$ is close to optimal, then $Z^T(Y - \hat{Y}) \approx 0$, $dA/dt \approx 0$ and $d^2A/dt^2 \approx 2R^TR$. Recall from Equation \ref{eq:FA_plus_A_update} that $dA/dt = 2R^TR$ for FA*. This $R^TR$ term is the main reason for alignment. Thus, when $Z^T(Y - \hat{Y}) \approx 0$, one can hope to argue that $A$ becomes more PSD with time and the ideas behind the analysis of FA* may be helpful for analysing FA. In general, if
one can understand the progression of $Z^T(Y - \hat{Y})$, combining it with our insights about FA* mays yield a rigorous understanding of FA alignment dynamics. 

\section{Understanding the stationary points}
\label{sec:stationary_points}
In the previous section, we saw that FA* converges to a stationary point satisfying $(Y-\hat{Y})C^T = 0$, where $\hat{Y}_{n \times m} = Z_{n \times r}W_{r \times m}$ and $W = (Z^TZ)^{-1}Z^TY$. In this section, we study the stationary points of feedback alignment and compare them to those of gradient flow (\ref{eq:gd_update}). From prior work \citep[Theorem 39, part (b)]{bah2019learning}, we know that gradient flow starting from a random initialization converges to the optimal solution almost surely. We investigate when the solution found by feedback alignment is optimal, and how different the representations found using feedback alignment and gradient flow can be.

Note that the stationary point equations are  same for FA and FA*, and are given by
\begin{align}
\begin{split}
\label{eq:stationary_FA}
    (Y - \hat{Y})C^T &= 0\\
    Z^T(Y - \hat{Y}) &= 0\\
    \hat{Y} &= ZW.
\end{split}
\end{align}
So the results of this section apply to both versions of feedback alignment.

\paragraph{Characterization of stationary points.}
In the next lemma, we characterize the solution $\hat{Y}$ at stationary points when $C$ is chosen randomly.
\begin{restatable}{lemma}{firstlemma}
\label{lem:FA_sol_char}
Suppose $C$ is chosen randomly with entries drawn i.i.d. from $\mathcal{N}(0,1)$ and the stationary point equations for feedback alignment (\ref{eq:stationary_FA}) are satisfied. Let $A_{n \times r}  = YC^T$ and $B_{r \times m} = \argmin_{B} \norm{A B - Y}_F^2$.  Then $\hat{Y} = AB$ almost surely.
\end{restatable}
The proof of Lemma \ref{lem:FA_sol_char} can be found in Appendix \ref{sec:sol_char_proof_FA_over_param_proof}.
To understand Lemma \ref{lem:FA_sol_char}, we write 
$Y = \sum_{i=1}^n \sigma_i u_i v_i^T$ where $\sigma_i$ is the $i^{\text{th}}$ singular value of $Y$, and $u_i$ and $v_i$ are the corresponding left and right singular vectors respectively. Then the $j^{\text{th}}$ column of $A$,  
\[A^{(j)} = \sum_{i=1}^n \sigma_i u_i R_{ij},\] where $R_{ij} = \langle v_i, C^{(j)} \rangle$ is a $\mathcal{N}(0,1)$ random variable.
That is, $A^{(j)}$ is a random linear combination of singular vectors of $Y$, scaled by its singular values. Lemma \ref{lem:FA_sol_char} says that feedback alignment finds the solution $\hat{Y}$ that corresponds to the best approximation of $Y$ (in Frobenius norm) in the space spanned by $A^{(j)}$s. On the other hand, gradient descent finds the solution that corresponds to the best approximation of $Y$ in the space spanned by top-$r$ singular vectors of $Y$, which is also the optimal solution (see e.g., \citep{blum2020foundations}).

\paragraph{Optimality of solution.}
Next, we show that when $r \geq rank(Y)$,  feedback alignment stationary points correspond to the optimal solution (see Fig. \ref{fig:separation_GD_FA_1} for an illustration).
\begin{restatable}{theorem}{overparamopt}
\label{thm:FA_over_param_opt}
Suppose $C$ is chosen randomly with entries drawn i.i.d. from $\mathcal{N}(0,1)$, the stationary point equations for feedback alignment (\ref{eq:stationary_FA}) are satisfied, and $r \geq rank(Y)$. Then $ZW = \hat{Y} = Y$ almost surely, which also minimizes $\norm{ZW - Y}_F^2.$
\end{restatable}
The proof of Theorem \ref{thm:FA_over_param_opt} follows directly from Lemma \ref{lem:FA_sol_char} and can be found in Appendix \ref{sec:sol_char_proof_FA_over_param_proof}. Columns of $A$ correspond to random linear combinations of singular vectors of $Y$, scaled by its singular values. We have at least $rank(Y)$ such columns. Therefore, the columns of $A$ span the singular vectors of $Y$ (corresponding to non-zero singular values) almost surely. Theorem  \ref{thm:FA_over_param_opt} follows since $\hat{Y}$ is the best approximation of $Y$ in the column span of $A$, which is equal to $Y$ almost surely.

Also, note that this result does not hold for arbitrary $C$. For instance, suppose $Y$ is a rank-$1$ matrix and $r=1$. Let $Z$ be any arbitrary full column-rank matrix, and $W = (Z^TZ)^{-1}Z^TY$.
In this case, $Y-\hat{Y}$ has rank at most $2$. If we choose $C_{1 \times m}$ such that its only row is orthogonal to the row space of $Y-\hat{Y}$, then $(Y-\hat{Y})C^T = 0$ and the stationary point equations (\ref{eq:stationary_FA}) are satisfied. However, $\hat{Y}$ may not be equal to $Y$. This motivates the random choice of $C$.

Next, we show that the feedback alignment solution can be far from optimal when $r$ is much smaller than $rank(Y)$ (see Fig.  \ref{fig:separation_GD_FA_2} for an illustration).
\begin{restatable}{theorem}{faseparationgd}
\label{thm:FA_separation_GD}
Suppose \\
(i)  $C$ has entries drawn i.i.d. from $\mathcal{N}(0,1)$, \\
(ii) the stationary point equations for feedback alignment (\ref{eq:stationary_FA}) are satisfied, \\
(iii) the singular values of $Y$ satisfy 
\[
    \sigma_i = \begin{cases}
        \frac{1}{\sqrt{2r}}, & \text{for \ } i \leq r\\
        \frac{1}{\sqrt{2(n-r)}}, & \text{for \ } r+1 \leq i \leq n
        \end{cases}
  \]\\
  (iv) $c_1 \leq  r \leq c_2 n$  for some absolute constants $c_1, \ c_2$, 
  
  then the  error $\norm{ZW - Y}_F^2 \geq 0.74$ with probability at least 0.99 over the choice of $C$. 
  
  On the other hand, gradient flow (\ref{eq:gd_update}) starting from randomly initialized $Z$ and $W$ ( with i.i.d. $\mathcal{N}(0, 1)$ entries) converges to the optimum solution with $\norm{ZW - Y}_F^2 = 0.5$ almost surely.
\end{restatable}
The proof of Theorem \ref{thm:FA_separation_GD} can be found in Appendix \ref{sec:FA_separation_GD_proof}. 

To understand Theorem \ref{thm:FA_separation_GD}, it is instructive to consider the case when $r=1$. Let the singular values of $Y$ satisfy
\[
    \sigma_i = \begin{cases}
        \frac{1}{\sqrt{2}}, & \text{for \ } i = 1\\
        \frac{1}{\sqrt{2(n-1)}}, & \text{for \ } 2 \leq i \leq n.
        \end{cases}
  \]
From Lemma \ref{lem:FA_sol_char}, we know that $\hat{Y} = A_{n \times 1} B_{1 \times m}$ almost surely. Here $A$ is a column vector satisfying
\[
A = \frac{1}{\sqrt{2}}u_1 R_1+ \frac{1}{\sqrt{2(n-1)}}\sum_{i=2}^n u_i R_{i},
\]
where $u_i$ are the left singular vectors of $Y$ and $R_i$ are drawn i.i.d. from $\mathcal{N}(0, 1)$.  On the other hand, the optimum solution corresponds to $A = u_1$. So while the optimum $A$ aligns with the top singular vector, the $A$ corresponding to feedback alignment has a significant component  in the orthogonal subspace. This causes the feedback alignment solution to be far from optimal.

\paragraph{Comparison of representations.}
In the previous result, we saw that the error achieved by the feedback alignment solution can be much higher than the gradient flow solution. Next, we demonstrate that even when the two errors are approximately equal, the representations recovered by the two algorithms can be almost orthogonal, again in the rank-deficient setting.
\begin{restatable}{theorem}{fagdorthrep}
\label{thm:FA_GD_rep_orthogonal}
Suppose\\
(i)  $C$ has entries drawn i.i.d. from $\mathcal{N}(0,1)$, \\
(ii) the stationary point equations for feedback alignment (\ref{eq:stationary_FA}) are satisfied, \\
(iii) the singular values of $Y$ satisfy $\sigma_1 = 1$ and $\sigma_i = \epsilon$ for $i > 1$, where $0 < \epsilon < 1$,\\
(iv) $r=1$ (rank 1 approximation) and $n \geq c$ for some absolute constant $c$.  
  
  Let $Z_{FA}$ and $W_{FA}$ denote the $Z$ and $W$ satisfying the above conditions respectively, and $Z_{GD}$ and $W_{GD}$ represent the factors found by gradient flow (\ref{eq:gd_update}) starting from randomly initialized $Z$ and $W$ (with i.i.d. $\mathcal{N}(0, 1)$ entries). Then 
  \[
  \norm{Z_{FA} W_{FA} - Y}_F^2 \leq \norm{Z_{GD} W_
  {GD}- Y}_F^2 \left(1 + \frac{2}{\epsilon^2 n}\right)
  \] 
  and 
  \[
  \left| \left\langle \frac{Z_{FA}}{\norm{Z_{FA}}_2}, \frac{Z_{GD}}{\norm{Z_{GD}}_2} \right\rangle \right| \leq \frac{4}{\epsilon\sqrt{n}}
  \]
  with probability at least $0.99$ over the choice of $C$ and random initialization of gradient flow.
\end{restatable}

The proof of Theorem \ref{thm:FA_GD_rep_orthogonal} can be found in Appendix \ref{sec:proof_FA_GD_rep_orthogonal}.

To understand Theorem \ref{thm:FA_GD_rep_orthogonal}, 
let us set $\epsilon = 0.5$. We get that the error of feedback alignment solution is at most $1+O\left(\frac{1}{n}\right)$ times that of the gradient flow solution, while $Z_{FA}$ and $Z_{GD}$ are almost orthogonal, with normalized absolute inner product $O\left(\frac{1}{\sqrt{n}}\right)$.

From Lemma \ref{lem:FA_sol_char}, we know that $\hat{Y} = A_{n \times 1} B_{1 \times m}$ almost surely. Here $A$ is a column vector satisfying
\[
A = u_1 R_1+ 0.5\sum_{i=2}^n u_i R_{i},
\]
where $u_i$ are the left singular vectors of $Y$ and $R_i$ are drawn i.i.d. from $\mathcal{N}(0, 1)$. Since $Z_{FA}$ and $A$ are column vectors and $Z_{FA}W_{FA} = AB$ almost surely, we get
\[
\frac{Z_{FA}}{\norm{Z_{FA}}_2} = \frac{u_1 R_1+ 0.5\sum_{i=2}^n u_i R_{i}}{\sqrt{R_1^2 + 0.25\sum_{i=2}^n R_i^2}}
\]
almost surely.
Since gradient flow converges to the optimum solution almost surely, we know
\[
\frac{Z_{GD}}{\norm{Z_{GD}}_2} = u_1
\]
almost surely. Using concentration of a $\chi$-squared random variable, we get that the normalized absolute inner product between $Z_{GD}$ and $Z_{FA}$ is $O\left(\frac{1}{\sqrt{n}}\right)$ with high probability.

The error of gradient flow solution $\norm{Z_{GD} W_
  {GD}- Y}_F^2$ is equal to the optimum error which is $\sum_{i=2}^n \sigma_i^2 = 0.25(n-1)$. It is not hard to see that the error of feedback alignment solution is at most $\norm{Y}_F^2$ which is equal to $1 + 0.25(n-1)$. From here, we get that the error of feedback alignment solution is at most $1+O\left(\frac{1}{n}\right)$ times that of the gradient flow solution.

Therefore, we get that the errors of feedback alignment solution and the gradient flow solution can be approximately equal, while the representations they find are almost orthogonal. We note that for low-rank matrix factorization, this phenomenon only occurs when the optimum error is large. For instance, when the optimum error is $0$, we are in the regime where $r \geq rank(Y)$. In this case, the column space of $Z_{FA}$ and $Z_{GD}$ are equal to the column space of $Y$ almost surely. It would be interesting to understand to what extent the representations found by gradient flow and feedback alignment are different for other problems such as for learning neural networks. And are there problems for which the 
 representations are significantly different even when the optimum error is small?

\section{LINEAR NEURAL NETWORKS}
\label{sec:lin_nn}

As a direct corollary, all our results for matrix factorization also hold for training two-layer linear neural networks, assuming the training input is isotropic and the output is a linear function of the input. Specifically, let $O = XY$, where the rows $X_i$ and $O_i$ represent the $i^{\text{th}}$ training input and output respectively, and let $X^TX = I$.  We want to find $Z$ and $W$ that minimize the training error
\begin{align}
\|XZW - O \|_F^2, 
\end{align}
which is equal to the matrix factorization error $\|ZW - Y \|_F^2$ (Equation \ref{eq:error_fa}). We use $X^TX = I$ here, and in the update equations below.
The gradient flow (GD) update for this problem is given by
\begin{align}
\label{eq:gd_update_nn}
\begin{split}
    \frac{dZ}{dt} &= X^T(O - XZW)W^T = (Y - \hat{Y})W^T \\
    \frac{dW}{dt} &= Z^TX^T(O - XZW) = Z^T(Y - \hat{Y}).
\end{split}
\end{align}
where $\hat{Y} = ZW$. The feedback alignment (FA) update is given by

\begin{align}
\label{eq:fa_update_nn}
\begin{split}
    \frac{dZ}{dt} &= X^T(O - XZW)C^T = (Y - \hat{Y})C^T \\
    \frac{dW}{dt} &= Z^TX^T(O - XZW) = Z^T(Y - \hat{Y}).
\end{split}
\end{align}
The update for feedback alignment with optimal $W$ (FA*) is given by

\begin{align}
\label{eq:fa_plus_update_nn}
\begin{split}
    \frac{dZ}{dt} &= X^T(O - XZW)C^T = (Y - \hat{Y})C^T \\
    W &= (Z^T X^T X Z)^{-1} Z^T X^T O =  (Z^TZ)^{-1}Z^TY
\end{split}
\end{align}

 As the GD, FA and FA* updates, and the error term are same as the corresponding updates and the error term for matrix factorization (Equations \ref{eq:gd_update}, \ref{eq:fa_update}, \ref{eq:fa_plus_update}), all our results also hold in this case.

\section{Conclusion}
We investigate how and when feedback alignment works, focusing on the problem of low-rank matrix factorization. For the ``how" question, we studied the dynamics of alignment between forward and backward weights, and its implications for convergence. For the ``when" question, we showed that feedback alignment converges to the optimal solution when the factorization has rank $r \geq rank(Y)$, but it can be far from optimal in the rank-deficient case where $r < rank(Y)$. To the best of our knowledge, this is the first rigorous separation result between feedback alignment and gradient descent. We also demonstrate that the representations learned by feedback alignment and gradient descent can be very different, even when their errors are approximately equal.

There are many interesting directions for future research. A natural next step is to extend our understanding of alignment dynamics to the problem of learning non-linear neural networks. \citet{song2021convergence} show that alignment may not happen in highly overparameterized neural networks. But it is a robust phenomenon in the parameter regimes typically encountered in practice, and therefore important to understand. It would also be interesting to understand the implicit regularization properties of feedback alignment and compare them to gradient descent by considering  problems such as matrix sensing in the overparameterized regime \citep{gunasekar2017implicit, li2018algorithmic}.
From the point of view of the theory of optimization, a fundamental question is whether feedback alignment is part of a larger family of algorithms (e.g., that replace parts of the gradient with random values) and whether it might be applicable to problems even without layered structure. More generally, building the mathematical foundations of biologically plausible learning is a fruitful direction that can reveal surprising algorithms while advancing our understanding of the brain.

\section{Acknowledgements}
We thank Pulkit Tandon, Rahul Trivedi and Tian Ye for helpful discussions. S.G. was supported by NSF awards AF-1813049 and AF-1704417, and a Stanford Interdisciplinary Graduate
Fellowship. S.S.V. was supported in part by NSF awards AF-1909756, AF-2007443 and AF-2134105. 

\bibliography{references}

\newpage
\appendix

\section{PROOF OF THEOREM \ref{thm:convergence-mat-fac}}
\label{sec:convg_proof}
Let $Y_{n \times m}$ be the matrix we want to factorize and $\hat{Y}_{n \times k} = Z_{n \times r}W_{r \times m}$. Let $C_{r \times m}$ be the feedback matrix.  Feedback alignment (FA*) updates $Z$ and $W$ as follows:
\begin{align*}
    \frac{dZ}{dt} &= (Y - \hat{Y})C^{T}\\
    W &= (Z^TZ)^{-1}Z^TY.
\end{align*}

\paragraph{Notation.} We will use $M(t)$ to denote  matrix $M$ at time $t$. However, we will not show time $t$ when it is clear from context. For a symmetric matrix $M$, we use $\lambda_i(M)$  to denote the $i^{\text{th}}$ largest eigenvalue of $M$. For any vector $v$, we use $\norm{v}$ to denote the $\ell_2$ norm of $v$. For any matrix $M$, we use $M^i$ to denote its $i^{\text{th}}$ row.  \\

We use $\mathds{1}[.]$ to denote the indicator function which is equal to $1$ if the condition inside the square brackets is true and $0$ otherwise.\\

We use $A$ to denote the alignment matrix $\left(\left(Z^TZ\right)^{-1}CW^T + WC^T\left(Z^TZ\right)^{-1}\right)$ and $R$ to denote the residual $(Y-\hat{Y})C^T$.\\

 For a non-zero vector $x(t)$, we use $x_{\leq k}(t)$ to denote vector $x(t)$ if $\frac{x(t)^T  A(t) x(t)}{\norm{x(t)}^2} \leq k$, and zero vector otherwise. Similarly, we use $x_{> k}(t)$ to denote vector $x(t)$ if $\frac{x(t)^T A(t) x(t)}{\norm{x(t)}^2} > k$, and zero vector otherwise.  We define $x_{\leq k}(t)$ and  $x_{> k}(t)$ to be equal to $x(t)$ when $x(t)$ is a zero vector.   For a matrix $M(t)$ with $i^{\text{th}}$ row $M^i(t)$, $M_{\leq k}(t)$ denotes the matrix whose $i^{th}$ row equals $M^i_{\leq k}(t)$ for all $i$. Similarly, we define $M_{> k}(t)$ to be the matrix whose $i^{th}$ row equals $M^i_{> k}(t)$ for all $i$.
Note that we can write $x(t) = x_{\leq k}(t) + x_{> k}(t)$ and $M(t) = M_{\leq k}(t) + M_{> k}(t)$.\\

We will use the loss function $\ell(t) = \norm{(Y-\hat{Y}(t))C^T \left(Z(t)^TZ(t)\right)^{-1/2}}_F^2$.

\thmconvg*

\paragraph{Note on the bound on T.}
The bound on $T$ depends on the condition number of $Z(0)$ and the top singular values of $Y$, $C$ and $Z(0)$. To understand this bound, suppose we set $\epsilon  = \epsilon_1 \norm{Y}_F^2 \norm{C}_F^2$ (for some $\epsilon_1 > 0$), so that $\epsilon$ has same scale as $\norm{(Y-\hat{Y})C^T}_F^2$. Then the bound on $T$ is given by 
\[
\frac{24}{\epsilon_1} \left( \frac{\sigma_1\left(Y\right)  \sigma_1\left(C\right)  \sigma_1\left(Z(0)\right)^6  \sqrt{r \ min(m, n)}}{\norm{Y}_F^2 \norm{C}_F^2  \sigma_r\left(Z(0)\right)^5 } \right).
\]
This bound decreases if we scale up $Y$ and $C$ and scale down $Z(0)$. This is because $\frac{dZ}{dt} = (Y - \hat{Y})C^T$. So the relative magnitude of update to $Z$ increases if we scale up $Y$ and $C$ and scale down $Z(0)$. The bound obtained would not be dependent on the scales of $Z(0)$, $C$, and $Y$, if we choose a scale independent  update given by $\frac{dZ}{dt} = \frac{(Y - \hat{Y})C^T \norm{Z}_F}{\norm{Y}_F \norm{C}_F}$.

\subsection{Proof Overview}
From Fact \ref{fact:Z_derivative}, we know that $Z^TZ$ does not change with time. While our formal proof holds for arbitrarily initialized $Z$ (with full column rank), in this proof sketch, we will assume that $Z$ is initialized such that $Z^TZ = I$.
In Lemma \ref{lem:dldt},  we show that \begin{align*}
\frac{d\ell}{dt} &= - Tr(RAR^T)\\
&= - \sum_{i=1}^n {R^i}^T A R^i
\end{align*}
where $R$ and $A$ are residual and alignment matrices respectively, as defined above, and $R^i$ is the $i^\text{th}$ row (viewed as a column vector) of $R$ . This implies that $\frac{d\ell}{dt} \leq 0$ if $A$ is PSD. In Lemma \ref{lem:dAdt}, we show that
\begin{align*}
\frac{dA}{dt} &= 2 (Z^TZ)^{-1} R^T R (Z^TZ)^{-1}\\
&= 2 R^T R.
\end{align*}
This implies that $x^T A x$ never decreases with time for all $x$. In this sense, $A$ becomes more PSD with time. However, this is not sufficient to claim that $A$ will become PSD eventually as there can exist $x$ for which $\frac{d  x^TAx}{dt} = 0$ at all times. Therefore, a more careful analysis of the directions in which $A$ becomes PSD ($x$ such that $x^T A x \geq 0$) is needed.

Whenever $\frac{d\ell}{dt}$ is positive, there is some row $R^i$ of $R$ for which ${R^i}^T A R^i$ is  negative. Also, note that $\frac{d \ x^TAx}{dt}  = 2\norm{R x}^2 \geq 2 ({R^i}^Tx)^2 > 0$, for $x$ such that ${R_i}^T x \neq 0$. So whenever a direction $R^i$ causes the loss to increase sufficiently,
$x^TAx$ also increases sufficiently for all $x$ close to $R^i$. And when $x^T A x > 0$ for all $x$, the loss can not increase anymore. That is, whenever some direction causes the loss to increase, $A$ becomes ``more PSD'' for all directions close to this direction, and when $A$ is PSD for all directions, the loss can not increase anymore. Using this idea, we bound the total increase in loss possible $\left(\int_0^T \frac{d\ell}{dt} \mathds{1} \left[\frac{d\ell}{dt} \geq 0\right] \ dt \right)$ in Lemma \ref{lem:loss-ub}. Using a similar idea, in lemma \ref{lem:loss-stable-ub}, we upper bound the total time for which the following holds: the loss is large and the loss is either increasing or decreasing slowly. At any other time, if the loss is large, it has to decrease sharply. Combining these two lemmas, in Lemma \ref{lem:FA-final-with-delta}, we show a bound on time by which the loss goes below $\epsilon$ at least once. In Lemma \ref{lem:FA-final}, we optimize the bound proved in Lemma \ref{lem:FA-final-with-delta}. In Lemma \ref{lem:FA-loss-final}, we translate the guarantee on $\ell(t)$ to the desired guarantee on $\norm{(Y-\hat{Y})C^T}_F^2$. The proves the first part of the theorem.

The results in Lemma \ref{lem:loss-ub} and \ref{lem:loss-stable-ub} crucially rely on Lemma \ref{lem:R-ub} which gives an upper bound on 
\[
\int_0^T \norm{R_{\leq k}(t)}_F^2 \ dt.
\]
for all $k \geq 0$ and for all $T$. This lemma helps formalize the intuition discussed above. An upper bound on $\int_0^T \norm{R_{\leq 0}(t)}_F^2 \ dt$ lets us upper bound the total increase in loss (Lemma \ref{lem:loss-stable-ub}). An upper bound on $\int_0^T \norm{R_{\leq k}(t)}_F^2 \ dt$ for positive $k$ lets us bound the total time for which the loss is large and is either increasing or decreasing slowly (Lemma \ref{lem:R-ub}).

Note that whenever $\norm{R_{\leq k}(t)x}_2^2$ is large for any $x$, $x^T A(t) x$ increases by a large amount. Also, by definition,  for any row $R^i_{\leq k}(T)$ (viewed as a column vector) of $R_{\leq k}(T)$, $R^i_{\leq k}(T)^ T A(T) R^i_{\leq k}(T)$ can not be too large, that is, $R^i_{\leq k}(T)^T A(T) R^i_{\leq k}(T) \leq k \norm{R^i_{\leq k}(T)}_2^2$. We also know that $R^i_{\leq k}(T)^T A(0) R^i_{\leq k}(T) \geq \lambda_r( A(0)) \norm{R^i_{\leq k}(T)}_2^2$.  Therefore, for any row $R^i_{\leq k}(T)$,
\begin{align*}
    \int_0 ^ T  \norm{R_{\leq k}(t) R^i_{\leq k}(T)}_2^2 dt &\leq \int_0 ^ T  \norm{R(t) R^i_{\leq k}(T)}_2^2 dt\\
    &= \frac{1}{2} \int_{0}^{T} R^i_{\leq k}(T)^T \frac{dA(t)}{dt} R^i_{\leq k}(T) dt\\
    &= \frac{1}{2} \left( R^i_{\leq k}(T)^T {A(T)} R^i_{\leq k}(T) - R^i_{\leq k}(T)^T {A(0)} R^i_{\leq k}(T) \right)\\
    &\leq \frac{1}{2} (k - \lambda_r(A(0))) \norm{R^i_{\leq k}(T)}_2^2
\end{align*}
In other words, this bounds the inner product of rows of $R_{\leq k}(T)$ with the rows of $R_{\leq k}(t)$ for $t \leq T$.
This fact lets us upper bound the integral of sum of squared norms of these rows. We prove such a bound for general vectors in Lemma \ref{lem:norm-bad-vec}, and use it to prove Lemma \ref{lem:R-ub}.  

We prove convergence of $\ell(t)$ to zero (second part of the theorem) in Lemma \ref{lem:convergence-last-iterate}, where we use the following argument. In Lemma \ref{lem:FA-final-with-delta}, we show that for all $\epsilon > 0$ and all $T \geq 0$, there exists $t \geq T$ such that $\ell(t) \leq \epsilon$. Now, if the loss doesn't converge to $0$, then there must exist some $\epsilon_1 > 0$, such that for all $T \geq 0$, there exists some $t \geq T$ satisfying $\ell(t) > \epsilon_1$. Using these two arguments, we can generate an infinite increasing sequence $T_1, T_1', T_2, T_2', \cdots$ such that $\ell(T_i) \leq \epsilon_1/2$ and $\ell(T_i') > \epsilon_1$ for all $i$. Thus we can get infinitely many disjoint intervals $[T_i, T_i']$ on which the loss increases by at least $\epsilon_1/2$, implying that the total increase in loss is unbounded which contradicts Lemma \ref{lem:loss-ub}, where we show $\int_0^T \frac{d\ell}{dt} \mathds{1} \left[\frac{d\ell}{dt} \geq 0\right]  dt $ is bounded for all $T$. Therefore, the loss $\ell(t)$ must converge to $0$. 

 
 \subsection{Proof}
 The following two facts will be useful for the proof.

\begin{fact}\label{fact:Z}
\[
Z^T(Y-\hat{Y})=0.
\]
\end{fact}

\begin{fact}\label{fact:Z_derivative}
\[
\frac{d(Z^TZ)}{dt}=0.
\]
\end{fact}
Fact \ref{fact:Z} follows since $\hat{Y} = ZW = Z(Z^TZ)^{-1}Z^T Y$. 
Fact \ref{fact:Z_derivative} follows since $\frac{d(Z^TZ)}{dt} = Z^T(Y - \hat{Y})C^T + (Z^T(Y - \hat{Y})C^T)^T  = 0$. 

Fact \ref{fact:Z_derivative} says that $Z^TZ$ does not change with time. While our result holds for arbitrarily initialized $Z$ (with full column rank), it might be helpful for the reader to assume that $Z$ is initialized such that $Z^TZ = I$. 

Now, we evaluate the expression for $\frac{d\ell}{dt}$.
\begin{lemma}
\label{lem:dldt}
\[
\frac{d\ell}{dt} = -Tr\left((Y-\hat{Y})C^T \left(\left(Z^TZ\right)^{-1}CW^T + WC^T\left(Z^TZ\right)^{-1}\right)C(Y - \hat{Y})^T\right).
\]
\end{lemma}
\begin{proof}
We can write 
\begin{align}
\label{eq:1-1}
    \frac{d\ell}{dt} = Tr\left(\left(\frac{dZ}{dt}\right)^{T}\left(\frac{d\ell}{dZ}\right)\right).
\end{align}
where
\begin{align*}
\frac{d\ell}{dZ} = -2(Y - \hat{Y})C^T(Z^TZ)^{-1}CW^T -2Z(Z^TZ)^{-1}C(Y-\hat{Y})^T(Y - \hat{Y})C^T(Z^TZ)^{-1}.
\end{align*}
Here, we used $Z^T(Y - \hat{Y}) = 0$ (see Fact \ref{fact:Z}) to simplify the expression.
Substituting this in equation \ref{eq:1-1}, we get
\begin{align*}
    \frac{d\ell}{dt} = Tr\left(\left(C(Y - \hat{Y})^T \right)\left(-2(Y - \hat{Y})C^T(Z^TZ)^{-1}CW^T -2Z(Z^TZ)^{-1}C(Y-\hat{Y})^T(Y - \hat{Y})C^T(Z^TZ)^{-1} \right)\right).
\end{align*}
Again using $Z^T(Y - \hat{Y}) = 0$, we get
\begin{align*}
    \frac{d\ell}{dt} &= -2Tr\left(C(Y - \hat{Y})^T  (Y - \hat{Y})C^T \left((Z^TZ)^{-1}CW^T\right)
    \right)
    \end{align*}
    Using the identities $Tr(MN) = Tr(NM)$ and $Tr(M^T) = Tr(M)$, we get
    \begin{align*}
    \frac{d\ell}{dt}
    &=-2Tr\left(  (Y - \hat{Y})C^T \left((Z^TZ)^{-1}CW^T\right) C(Y - \hat{Y})^T \right)\\
    &= -Tr\left((Y-\hat{Y})C^T \left(\left(Z^TZ\right)^{-1}CW^T + WC^T\left(Z^TZ\right)^{-1}\right)C(Y - \hat{Y})^T\right).
\end{align*}
\end{proof}

Recall that the alignment matrix is $A = \left(\left(Z^TZ\right)^{-1}CW^T + WC^T\left(Z^TZ\right)^{-1}\right)$ and the residual is $R = (Y-\hat{Y})C^T$.
Next, we show how $A$ changes with time.

\begin{lemma}
\label{lem:dAdt}
\begin{align*}
\frac{dA}{dt} = 2 (Z^TZ)^{-1} R^T R (Z^TZ)^{-1}
\end{align*}
\end{lemma}
\begin{proof}
From Fact \ref{fact:Z_derivative}, we know that $\frac{d(Z^TZ)}{dt} = 0$.
This implies
\begin{align}
\label{eq:dAdt1}
\frac{dA}{dt} = \left(\left(Z^TZ\right)^{-1}C\frac{dW}{dt}^T + \frac{dW}{dt}C^T\left(Z^TZ\right)^{-1}\right)
\end{align}
Here, $W = (Z^TZ)^{-1}Z^TY$. Again, using $\frac{d(Z^TZ)}{dt} = 0$ and $\frac{dZ}{dt} = (Y-\hat{Y})C^T$, we get
\[
\frac{dW}{dt} = (Z^TZ)^{-1}C(Y-\hat{Y})^TY.
\]
From Fact \ref{fact:Z}, we know that $Z^T(Y - \hat{Y}) = 0$ which implies $(Y-\hat{Y})^T\hat{Y} = 0$. Using this, we get
\[
\frac{dW}{dt} = (Z^TZ)^{-1}C(Y-\hat{Y})^T(Y-\hat{Y}).
\]
Substituting this in Equation \ref{eq:dAdt1}, we get 
\begin{align*}
\frac{dA}{dt} 
&= 2 (Z^TZ)^{-1}  C(Y-\hat{Y})^T(Y-\hat{Y})C^T (Z^TZ)^{-1}\\
&= 2 (Z^TZ)^{-1} R^T R (Z^TZ)^{-1}.
\end{align*}

\end{proof}

The lemma below essentially says the following: suppose we observe $n$ vectors $v_1(t), v_2(t) \cdots v_n(t)$ at each time $t$ , and let the vectors observed at any time $T$ have small inner product with all vectors observed before time $T$ ($\sum_{i=1}^n \int_0^{T} \langle v_k(T), v_i(t) \rangle^2 dt$ is small),for all $T$. Then $\sum_{i=1}^n \ \int_0^{T} \norm{v_i(t)}^2 dt$ can not be too large, for all $T$.

\begin{lemma}
\label{lem:norm-bad-vec}
Let $v_i(t):\R \rightarrow \R^r$ for $i \in \{1, 2 \cdots, n\}$ such that 
\begin{equation}
\label{eq:cond_inner_prod}
\sum_{i=1}^n \int_0^{T} \langle v_k(T), v_i(t) \rangle^2 dt \leq c \norm{v_k(T)}^2
\end{equation}
for all $T, k$ and for some constant $c \geq 0$. Then 
\begin{equation}
\sum_{i=1}^n \ \int_0^{T} \norm{v_i(t)}^2 dt \leq 2  r c,
\end{equation}
for all $T$ (assuming $\norm{v_i(t)}^2$ is integrable).
\end{lemma}
\begin{proof}
Let $L_T = \sum_{i=1}^n \ \int_0^{T} v_i(t)v_i(t)^T dt$. Sum of eigenvalues of $L$ is given by 
\begin{equation}
\label{eq:norm-bad-vec_eig-bound}
    \sum_{i=1}^r \lambda_i(L_T) 
    = \sum_{i=1}^n \ \int_0^{T} \norm{v_i(t)}^2 dt.
\end{equation}
Here, we used $\sum_{i=1}^r \lambda_i(L_T)  = Tr(L_T)$. Using $\sum_{i=1}^r \lambda_i(L_T)^2  = Tr(L_TL_T^T)$, we get
\begin{equation*}
\begin{split}
    \sum_{i=1}^r \lambda_i(L_T)^2 
    &= \sum_{i=1}^n \sum_{j=1}^n \ \int_0^{T} \int_0^{T} \langle v_i(t_1), v_j(t_2) \rangle^2 \ dt_1 \ dt_2\\
    &= \sum_{i=1}^n \sum_{j=1}^n  \ \int_0^{T}  \int_0^{t_2} \langle v_i(t_1), v_j(t_2) \rangle^2 \ dt_1 \ dt_2 + \sum_{i=1}^n \sum_{j=1}^n  \ \int_0^{T}  \int_0^{t_1} \langle v_i(t_1), v_j(t_2) \rangle^2 \ dt_2 \ dt_1\\
    &= \sum_{j=1}^n   \ \int_0^{T} \sum_{i=1}^n \int_0^{t_2} \langle v_i(t_1), v_j(t_2) \rangle^2 \ dt_1 \ dt_2 +  \sum_{i=1}^n  \ \int_0^{T} \sum_{j=1}^n  \int_0^{t_1} \langle v_i(t_1), v_j(t_2) \rangle^2 \ dt_2 \ dt_1
    \end{split}
\end{equation*}
Using condition \ref{eq:cond_inner_prod}, we get 
\begin{equation}
\label{eq:norm-bad-vec_eig-sq-bound}
    \sum_{i=1}^r \lambda_i(L_T)^2 \leq 2 \sum_{i=1}^n \ \int_0^T c\norm{v_i(t)}^2 dt.
\end{equation}
Using Cauchy–Schwarz inequality, we can write
\begin{equation*}
    \left( \sum_{i=1}^r \lambda_i(L_T) \right)^2 \leq r \sum_{i=1}^r \lambda_i(L_T)^2.
\end{equation*}
Substituting from Equation \ref{eq:norm-bad-vec_eig-bound} and \ref{eq:norm-bad-vec_eig-sq-bound}, we get
\[
\left( \sum_{i=1}^n \ \int_0^{T} \norm{v_i(t)}^2 dt \right)^2 \leq 2rc  \sum_{i=1}^n \ \int_0^T \norm{v_i(t)}^2 dt.
\]
which implies
\[
 \sum_{i=1}^n \ \int_0^{T} \norm{v_i(t)}^2 dt  \leq 2rc  .
\]
In the last step, we used $\sum_{i=1}^n \ \int_0^{T} \norm{v_i(t)}^2 dt \neq 0$. If it is equal to zero, then the lemma is trivially true.
\end{proof}

 Recall that $A$ denotes the alignment matrix $\left(\left(Z^TZ\right)^{-1}CW^T + WC^T\left(Z^TZ\right)^{-1}\right)$, and $R$ denotes the residual matrix $(Y - \hat{Y})C^T$. 
 
 Also recall the following notation.
 For a non-zero vector $x(t)$, we use $x_{\leq k}(t)$ to denote vector $x(t)$ if $\frac{x(t)^T  A(t) x(t)}{\norm{x(t)}^2} \leq k$, and zero vector otherwise. Similarly, we use $x_{> k}(t)$ to denote vector $x(t)$ if $\frac{x(t)^T A(t) x(t)}{\norm{x(t)}^2} > k$, and zero vector otherwise.  We define $x_{\leq k}(t)$ and  $x_{> k}(t)$ to be equal to $x(t)$ when $x(t)$ is a zero vector.   For a matrix $M(t)$ with $i^{\text{th}}$ row $M^i(t)$, $M_{\leq k}(t)$ denotes the matrix whose $i^{th}$ row equals $M^i_{\leq k}(t)$ for all $i$. Similarly, we define $M_{> k}(t)$ to be the matrix whose $i^{th}$ row equals $M^i_{> k}(t)$ for all $i$.
Note that we can write $x(t) = x_{\leq k}(t) + x_{> k}(t)$ and $M(t) = M_{\leq k}(t) + M_{> k}(t)$.
 
 In the next lemma, we show an upper bound on $\int_0^T \norm{R_{\leq k}(t)}_F^2 \ dt$. We will see in Lemma \ref{lem:loss-ub} and Lemma \ref{lem:loss-stable-ub} that  $\norm{R_{\leq k}(t)}_F^2 $ being large corresponds to certain undesirable events. For example, in Lemma \ref{lem:loss-ub}, we will see that $\norm{R_{\leq 0}(t)}_F^2 $ being large corresponds to increase in loss $\ell(t)$. The next lemma will be helpful in bounding the total time for which such undesirable events can happen.
\begin{lemma}
\label{lem:R-ub}
For all $k \geq \lambda_r\left(A(0)\right)$ and for all $T$,
\[
\int_0^T \norm{R_{\leq k}(t)}_F^2 \ dt \leq  \ r \ \left( k-\lambda_r\left(A(0)\right)\right) \ \lambda_1\left(Z(0)^TZ(0)\right)^2.
\]
\end{lemma}
\begin{proof}
Let $R^i_{\leq k}(T)$ be any row of $R_{\leq k}(T)$.
By definition, we know that $R^i_{\leq k}(T)$ (viewed as a column vector) satisfies 
\begin{equation}
{({R^i_{\leq k}(T)})^T A(T) (R^i_{\leq k}(T))} \leq k {\norm{R^i_{\leq k}(T)}^2}
\end{equation}
for all $i$. Since $\lambda_r(A(0))$ is the minimum eigenvalue of $A(0)$, we also know that
\begin{equation}
        {({R^i_{\leq k}(T)})^T A(0) (R^i_{\leq k}(T))} \geq \lambda_r(A(0)) \ {\norm{R^i_{\leq k}(T)}^2}.
\end{equation}
The last two equation imply
\begin{equation}
    {({R^i_{\leq k}(T)})^T A(T) (R^i_{\leq k}(T))}
    -
    {({R^i_{\leq k}(T)})^T A(0) (R^i_{\leq k}(T))} \leq \left(k-\lambda_r(A(0))\right) \ {\norm{R^i_{\leq k}(T)}^2}.
\end{equation}
\\This implies
\begin{equation}
\int_0^T {({R^i_{\leq k}(T)})^T \left(\frac{d (A(t))}{dt} \right) (R^i_{\leq k}(T))} dt \leq \left(k-\lambda_r(A(0))\right) {\norm{R^i_{\leq k}(T)}^2}.
\end{equation}
Substituting for $\frac{d(A(t))}{dt}$ from Lemma \ref{lem:dAdt}, we get
\begin{equation}
    \int_0^T {2 ({R^i_{\leq k}(T)})^T \left(\left(Z(t)^TZ(t)\right)^{-1} R(t)^TR(t) \left(Z(t)^TZ(t)\right)^{-1} \right) (R^i_{\leq k}(T))} \ dt \leq \left(k-\lambda_r(A(0))\right) {\norm{R^i_{\leq k}(T)}^2}.
\end{equation}
\\Note that for all $x$ and $k$, $x^T R^T R x \geq x^T R_{\leq k}^T R_{\leq k} x$. Also using Fact \ref{fact:Z_derivative}, we know that $\left(Z(t)^TZ(t)\right)^{-1} = \left(Z(0)^TZ(0)\right)^{-1}$. This gives us
\begin{equation}
    \int_0^T {({R^i_{\leq k}(T)})^T \left(\left(Z(0)^TZ(0)\right)^{-1} R_{\leq k}(t)^TR_{\leq k}(t) \left(Z(0)^TZ(0)\right)^{-1} \right) (R^i_{\leq k}(T))} \ dt \leq \frac{\left(k-\lambda_r(A(0))\right) \ {\norm{R^i_{\leq k}(T)}^2}}{2} .
\end{equation}
The expression in the above integral is integrable by Lemma \ref{lem:Rless_integrable}.

Now, define $v_i(t) = \left(Z(0)^TZ(0)\right)^{-\frac{1}{2}}{R^i_{\leq k}(t)} $. The above equation implies
\begin{equation}
    \int_0^T {\sum_{j=1}^n \langle v_i(T), v_j(t) \rangle^2 } \ dt \leq \frac{\left(k-\lambda_r(A(0))\right) \ {v_i(T)^T(Z(0)^TZ(0))v_i(T)}}{2}.
\end{equation}
We know $v_i(T)^T(Z(0)^TZ(0))v_i(T) \leq \lambda_1(Z(0)^TZ(0)) \norm{v_i(T)}^2$. This gives us 
\begin{equation}
    \int_0^T \sum_{j=1}^n \langle v_i(T), v_j(t) \rangle^2 \ dt \leq \frac{\left( k-\lambda_r\left(A(0)\right)\right) \  \lambda_1\left(Z(0)^TZ(0)\right) \ \norm{v_i(T)}^2}{2}.
\end{equation}
Using Lemma \ref{lem:norm-bad-vec}, we get 
\begin{equation}
    \int_0^T \sum_{j=1}^n \norm{v_j(t)}^2 \ dt \leq  \ r \ \left( k-\lambda_r\left(A(0)\right)\right) \ \lambda_1\left(Z(0)^TZ(0)\right),
\end{equation}
for all $T$ ($ \norm{v_j(t)}^2$ is integrable by Lemma \ref{lem:Rless_integrable}).
This is equivalent to
\begin{equation}
\label{eq:R-ub_2}
    \int_0^T \norm{{R_{\leq k}(t)} \left(Z(0)^TZ(0)\right)^{-\frac{1}{2}} }_F^2 \ dt \leq  \ r \ \left( k-\lambda_r\left(A(0)\right)\right) \ \lambda_1\left(Z(0)^TZ(0)\right),
\end{equation}
for all $T$.
We know 
\begin{equation}
    \begin{split}
        \norm{{R_{\leq k}(t)} \left(Z(0)^TZ(0)\right)^{-\frac{1}{2}} }_F^2 &\geq \norm{{R_{\leq k}(t)}}_F^2 \ \lambda_r
\left(\left(Z(0)^TZ(0)\right)^{-\frac{1}{2}} \right)^2\\
&= \frac{\norm{{R_{\leq k}(t)}}_F^2}{ \lambda_1
\left(\left(Z(0)^TZ(0)\right) \right)}
    \end{split}
\end{equation}
Combining this with Equation \ref{eq:R-ub_2}, we get
\begin{equation}
    \int_0^T \norm{{R_{\leq k}(t)}  }_F^2 \ dt \leq  \ r \ \left( k-\lambda_r\left(A(0)\right)\right) \ \lambda_1\left(Z(0)^TZ(0)\right)^2,
\end{equation}
for all $T$ ($ \norm{{R_{\leq k}(t)}  }_F^2$ is integrable by Lemma   \ref{lem:Rless_integrable}).

This comples the proof.
\end{proof}
Recall, we use $\mathds{1}[.]$ to denote the indicator function which is equal to $1$ if the condition inside the square brackets is true and $0$ otherwise.

In the next lemma, we show an upper bound on total increase possible in loss $\ell(t)$.
\begin{lemma}
\label{lem:loss-ub}
For all $T$,
\[
\int_0^T \ \frac{d  \ell(t)}{dt} \  \mathds{1}\left[\frac{d  \ell(t)}{dt} \geq 0 
\right] \ dt  \leq  r \ \lambda_r\left(A(0)\right)^2 \ \lambda_1\left(Z(0)^TZ(0)\right)^2 .
\]
\end{lemma}
\begin{proof}
We do not explicitly show time $t$ for variables in the proof below. Whenever time is not written, the corresponding variable is evaluated at some arbitary time $t$.

From Lemma \ref{lem:dldt} and Lemma \ref{lem:dAdt}, we know that
\begin{equation}
\frac{d\ell}{dt} = -Tr(RAR^T)
\end{equation}
\begin{equation}
\frac{d (x^TAx)}{dt} = 2 \ x^T(Z^TZ)^{-1}R^TR(Z^TZ)^{-1}x
\end{equation}
for all fixed $x$ (which do not change with $t$). Here, $R$ denotes the residual matrix $(Y-\hat{Y})C^T$ and $A$  denotes the alignment matrix $\left(\left(Z^TZ\right)^{-1}CW^T + WC^T\left(Z^TZ\right)^{-1}\right)$. 

Note that if $\lambda_r(A(0)) > 0$, then $\frac{d\ell}{dt} < 0$ for all $t$ and the lemma is trivially true. In the proof below, we assume $\lambda_r(A(0)) \leq 0$.


By the definition of $R_{\leq k}$ and $R_{>k}$, we can write
\begin{equation}
\frac{d\ell}{dt} = -\left(Tr(R_{\leq 0}AR_{\leq 0}^T)+ Tr(R_{> 0}AR_{> 0}^T)\right). 
\end{equation}
By definition of $R_{>0}$, we know that $Tr(R_{> 0}AR_{> 0}^T) \geq 0$. This gives us
\begin{equation}
\begin{split}
    \frac{d\ell}{dt} &\leq -Tr(R_{\leq 0}AR_{\leq 0}^T)\\
&\leq -\lambda_r(A) \norm{R_{\leq 0}}_F^2.
\end{split}
\end{equation}
Here, $\lambda_r(A)$ is the minimum eigenvalue of $A(t)$. Since $\frac{d(x^T A x)}{dt} \geq 0$ for all fixed $x$, the minimum eigenvalue of $A$ never decreases. That is, $\lambda_r(A(t)) \geq \lambda_r(A(0))$. This implies
\begin{equation}
\label{eq:loss-ub_1}
    \frac{d\ell}{dt} \leq -\lambda_r(A(0)) \norm{R_{\leq 0}}_F^2.
\end{equation}
From Lemma \ref{lem:R-ub}, we know
\begin{equation}
    \int_0^T \norm{{R_{\leq 0}(t)}  }_F^2 \ dt \leq - \ r \ \lambda_r\left(A(0)\right) \ \lambda_1\left(Z(0)^TZ(0)\right)^2,
\end{equation}
for all $T$.
Therefore we get
\begin{equation}
\begin{split}
    \int_0^T \ \frac{d  \ell(t)}{dt} \  \mathds{1}\left[\frac{d  \ell(t)}{dt} \geq 0 
\right] \ dt & \leq  \int_0^T -\lambda_r(A(0)) \norm{R_{\leq 0}(t)}_F^2 \mathds{1}\left[\frac{d  \ell(t)}{dt} \geq 0 
\right]  dt\\ 
&\leq -\lambda_r(A(0)) \ \int_0^T  \norm{R_{\leq 0}(t)}_F^2  dt\\
    &\leq r \ \lambda_r\left(A(0)\right)^2 \ \lambda_1\left(Z(0)^TZ(0)\right)^2,
    \end{split}
\end{equation}
for all $T$.
Here, we used $\lambda_r(A(0)) \leq 0$.  The expressions in the above integral are Lebesgue integrable as product of bounded integrable functions is integrable (over any finite interval). Here, $\frac{d\ell}{dt}$ is integrable due to continuity, $\mathds{1}\left[\frac{d  \ell(t)}{dt} \geq 0\right]$ is integrable by Lemma \ref{lem:leb_int_ind_contfunc} and $\norm{R_{\leq 0}(t)}_F^2$ is integrable by Lemma \ref{lem:Rless_integrable}.

This completes the proof.
\end{proof}

In the next lemma, we bound the total time for which loss $\norm{R}_F$ is large, and $\ell$ is either increasing, or decreasing slowly. Since $\ell = \norm{R (Z^TZ)^{-1/2}}_F^2$ and $Z^TZ$ doesn't change with time, this also gives a bound for total time for which loss $\ell$ is large, and $\ell$ is either increasing, or decreasing slowly.
\begin{lemma}
\label{lem:loss-stable-ub}
For all $\delta > 0, \ \epsilon > 0$ and for all $T$,
\[
\int_0^T \mathds{1} \left[\frac{d\ell(t)}{dt} > -\delta \text{ and } \norm{R(t)}_F^2 > \epsilon \right] \ dt \leq \ \frac{r \ \left( \frac{2\delta}{\epsilon}-\lambda_r\left(A(0)\right)\right)^2 \ \lambda_1\left(Z(0)^TZ(0)\right)^2}{\delta}.
\]
\end{lemma}
\begin{proof}
We do not explicitly show time $t$ for variables in the proof below. Whenever time is not written, the corresponding variable is evaluated at some arbitary time $t$.

From Lemma \ref{lem:dldt} and Lemma \ref{lem:dAdt}, we know that
\begin{equation}
\frac{d\ell}{dt} = -Tr(RAR^T),
\end{equation}
\begin{equation}
\frac{d (x^TAx)}{dt} = 2 \ x^T(Z^TZ)^{-1}R^TR(Z^TZ)^{-1}x.
\end{equation}
for all fixed $x$ (which do not change with $t$). Here, $R$ denotes the residual matrix $(Y-\hat{Y})C^T$ and $A$  denotes the alignment matrix $\left(\left(Z^TZ\right)^{-1}CW^T + WC^T\left(Z^TZ\right)^{-1}\right)$.

Note that if $\lambda_r(A(0)) > \frac{2\delta}{\epsilon}$, then $\frac{d\ell}{dt} < -2 \delta$ for all $t$ where $\norm{R(t)}_F^2 \geq \epsilon$. In this case, the lemma is trivially true. In the proof below, we assume $\lambda_r(A(0)) \leq \frac{2\delta}{\epsilon}$.

Let $\mathds{1} \left[\frac{d\ell(t)}{dt} > -\delta \text{ and } \norm{R(t)}_F^2 > \epsilon \right] = 1$ at the current time $t$.  Let $k = \frac{2\delta}{\epsilon}$. By the definition of $R_{\leq k}$ and $R_{>k}$, we can write
\begin{equation}
\frac{d\ell}{dt} = -\left(Tr(R_{\leq k}AR_{\leq k}^T)+ Tr(R_{> k}AR_{> k}^T)\right). 
\end{equation}
Since $\frac{d\ell}{dt} > -\delta$, we get
\begin{equation}
\label{eq:loss-stable-ub_1}
    \left(Tr(R_{\leq k}AR_{\leq k}^T)+ Tr(R_{> k}AR_{> k}^T)\right) < \delta.
\end{equation}
By definition of $R_{>k}$, we know that
\begin{equation}
\label{eq:loss-stable-ub_2}
 Tr(R_{> k}AR_{> k}^T) \geq \norm{R_{> k}}_F^2k .
\end{equation}
Since $\lambda_r(A)$ is the minimum eigenvalue of $A(t)$ and since the minimum eigenvalue of $A$ does not decrease with time, we can write
\begin{equation}
\label{eq:loss-stable-ub_3}
\begin{split}
 Tr(R_{\leq k}AR_{\leq k}^T) &\geq \norm{R_{\leq k}}_F^2 \ \lambda_r(A)\\
 &\geq \norm{R_{\leq k}}_F^2 \ \lambda_r(A(0)).
 \end{split}
\end{equation}
 Equation \ref{eq:loss-stable-ub_1}, \ref{eq:loss-stable-ub_2} and \ref{eq:loss-stable-ub_3} together imply
 \begin{equation}
     \norm{R_{\leq k}}_F^2 \ \lambda_r(A(0)) + \norm{R_{> k}}_F^2k   \leq \delta.
 \end{equation}
 Since $\norm{R_{\leq k}}_F^2 + \norm{R_{> k}}_F^2 > \epsilon$ and $k = \frac{2\delta}{\epsilon}$, we get
 \begin{equation}
     \norm{R_{\leq k}}_F^2 \ \lambda_r(A(0)) + \left(\epsilon - \norm{R_{\leq k}}_F^2  \right) \frac{2\delta}{\epsilon} \leq \delta.
 \end{equation}
Rearranging the terms, we get 
\begin{equation}
    \norm{R_{\leq k}}_F^2 \geq \frac{\delta}{\frac{2\delta}{\epsilon} - \lambda_r(A(0))}.
\end{equation}
This implies
\begin{equation}
    \begin{split}
        \int_0^T \norm{R_{\leq k}(t)}_F^2 \ dt &\geq  \int_0^T \norm{R_{\leq k}(t)}_F^2 \mathds{1} \left[\frac{d\ell(t)}{dt} > -\delta \text{ and } \norm{R(t)}_F^2 > \epsilon \right] \ dt\\
        &\geq \frac{\delta}{\frac{2\delta}{\epsilon} - \lambda_r(A(0))} \int_0^T \mathds{1} \left[\frac{d\ell(t)}{dt} > -\delta \text{ and } \norm{R(t)}_F^2 > \epsilon \right] \ dt,
    \end{split}
\end{equation}
for all $T$. The expressions in the above integral are Lebesgue integrable as product of bounded integrable functions is integrable (over any finite interval). Here,  $\mathds{1}\left[\frac{d  \ell(t)}{dt} > -\delta \right]$ and $\mathds{1}\left[\norm{R(t)}_F^2 > \epsilon \right]$ are integrable by Lemma \ref{lem:leb_int_ind_contfunc} and $\norm{R_{\leq k}(t)}_F^2$ is integrable by Lemma \ref{lem:Rless_integrable}.

From Lemma \ref{lem:R-ub}, we know 
\begin{equation}
\begin{split}
\int_0^T \norm{R_{\leq k}(t)}_F^2 \ dt &\leq  \ r \ \left( k-\lambda_r\left(A(0)\right)\right) \ \lambda_1\left(Z(0)^TZ(0)\right)^2\\
&= \ r \ \left( \frac{2\delta}{\epsilon}-\lambda_r\left(A(0)\right)\right) \ \lambda_1\left(Z(0)^TZ(0)\right)^2,
\end{split}
\end{equation}
for all $T$.
From last two equations, we conclude
\begin{equation}
    \int_0^T \mathds{1} \left[\frac{d\ell(t)}{dt} > -\delta \text{ and } \norm{R(t)}_F^2 > \epsilon \right] \ dt \leq \ \frac{r \ \left( \frac{2\delta}{\epsilon}-\lambda_r\left(A(0)\right)\right)^2 \ \lambda_1\left(Z(0)^TZ(0)\right)^2}{\delta},
\end{equation}
for all $T$.
\end{proof}

Using above lemmas, in the next lemma, we show that the loss goes below $\epsilon$ at least once in every length $T$ time interval for appropriately defined $T$.

\begin{lemma}
\label{lem:FA-final-with-delta}
For all $\epsilon > 0$, $\delta > 0$, $T_1 \geq 0$, and
\[
T \geq \frac{r  \left( \frac{2\delta}{\epsilon \lambda_r(Z(0)^TZ(0))}-\lambda_r\left(A(0)\right)\right)^2 \ \lambda_1\left(Z(0)^TZ(0)\right)^2}{\delta} + \frac{\ell(T_1) + \ r \ \lambda_r\left(A(0)\right)^2 \ \lambda_1\left(Z(0)^TZ(0)\right)^2  - \epsilon}{\delta} ,
\] FA dynamics satisfy 
\[
\min_{T_1 \leq t \leq T_1+T} \ \ell(t) \leq \epsilon. 
\]
\end{lemma}
\begin{proof}
Assume $\ell(t) > \epsilon$ for all $T_1 \leq t< T_1 + T$. We would show that $\ell(T_1 + T) \leq \epsilon$ in this case, which would imply the lemma.

We can write
\begin{equation}
\label{eq:FA-final-with-delta_1}
    \begin{split}
        \ell({T_1 + T}) &= \ell(T_1) + \int_{T_1}^{T_1 + T} \frac{d\ell(t)}{dt} \ dt\\
        &= \ell(T_1 ) + \int_{T_1} ^{T_1 + T} \frac{d\ell(t)}{dt} \ \mathds{1}\left[\frac{d\ell(t)}{dt} \leq -\delta \right] \ dt + \ \int_{T_1 }^{T_1 + T} \frac{d\ell(t)}{dt} \ \mathds{1}\left[\frac{d\ell(t)}{dt} > -\delta \right] \ dt\\
        &\leq \ell(T_1 ) + \int_{T_1} ^{T_1 + T} \frac{d\ell(t)}{dt} \ \mathds{1}\left[\frac{d\ell(t)}{dt} \leq -\delta \right] \ dt + \ \int_{T_1 }^{T_1 + T} \frac{d\ell(t)}{dt} \ \mathds{1}\left[\frac{d\ell(t)}{dt} \geq 0 \right] \ dt.
    \end{split}
\end{equation}
The expressions in the above integral are Lebesgue integrable as product of bounded integrable functions is integrable (over any finite interval). Here,  $\frac{d\ell}{dt}$ is integrable by continuity, $\mathds{1}\left[\frac{d  \ell(t)}{dt} > -\delta \right]$, $\mathds{1}\left[\frac{d  \ell(t)}{dt} \geq 0\right]$ and $\mathds{1}\left[\frac{d  \ell(t)}{dt} \leq -\delta\right]$ are integrable by Lemma \ref{lem:leb_int_ind_contfunc}.

We will now bound the two integral terms in the RHS. From Lemma \ref{lem:loss-ub}, we know
\begin{align}
\label{eq:FA-final-with-delta_2}
  \int_{T_1 }^{T_1 + T} \frac{d\ell(t)}{dt} \ \mathds{1}\left[\frac{d\ell(t)}{dt} \geq 0 \right] \ dt &\leq
    \int_0^{T_1+T} \ \frac{d  \ell(t)}{dt} \  \mathds{1}\left[\frac{d  \ell(t)}{dt} \geq 0 
\right] \ dt \\
&\leq  r \ \lambda_r\left(A(0)\right)^2 \ \lambda_1\left(Z(0)^TZ(0)\right)^2 
\end{align}
We can write
\begin{equation}
\label{eq:FA-final-with-delta_3}
    \begin{split}
        \int_{T_1}^{T_1+T} \frac{d\ell(t)}{dt} \ \mathds{1}\left[\frac{d\ell(t)}{dt} \leq -\delta \right] \ dt &\leq -\delta \int_{T_1}^{T_1+T}  \mathds{1}\left[\frac{d\ell(t)}{dt} \leq -\delta \right] \ dt\\
        &= -\delta \left(T - \int_{T_1}^{T_1+T}  \mathds{1}\left[\frac{d\ell(t)}{dt} > -\delta \right] \ dt  \right)\\
        &= -\delta \left(T - \int_{T_1}^{T_1+T}  \mathds{1}\left[\frac{d\ell(t)}{dt} > -\delta \text{ and } \ell(t) > \epsilon \right] \ dt  \right)
    \end{split}
\end{equation}
where we used the assumption $\ell(t) > \epsilon$ for all $T_1 \leq t < T_1+ T$ in the last equation. Recall that $\ell(t) = \norm{R(t)(Z(t)^TZ(t))^{-\frac{1}{2}}}_F^2 = \norm{R(t)(Z(0)^TZ(0))^{-\frac{1}{2}}}_F^2 $. Therefore, we can write
\begin{equation}
    \begin{split}
        \epsilon \ &< \ \norm{R(t)(Z(0)^TZ(0))^{-\frac{1}{2}}}_F^2 \\
        &< \ \norm{R(t)}_F^2 \ \lambda_1((Z(0)^TZ(0))^{-1})\\
        &= \frac{\norm{R(t)}_F^2}{\lambda_r(Z(0)^TZ(0))}
    \end{split}
\end{equation}
Using this, we can write

\begin{align}
    \begin{split}
\int_{T_1}^{T_1+T} \mathds{1}\left[\frac{d\ell(t)}{dt} > -\delta \text{ and } \ell(t) > \epsilon \right] \ dt &\leq     
\int_{0}^{T_1 + T} \mathds{1}\left[\frac{d\ell(t)}{dt} > -\delta \text{ and } \ell(t) > \epsilon \right] \ dt\\
&\leq
\int_{0}^{T_1+T} \mathds{1}\left[\frac{d\ell(t)}{dt} > -\delta \text{ and } \norm{R(t)}_F^2 > \epsilon \lambda_r(Z(0)^TZ(0)) \right] \ dt.
    \end{split}
\end{align}
The expressions in the above integral are Lebesgue integrable as product of bounded integrable functions is integrable (over any finite interval). Here, $\mathds{1}\left[\frac{d  \ell(t)}{dt} > -\delta \right]$, $\mathds{1}\left[l(t ) >  \epsilon \right]$ and $\mathds{1} \left[\norm{R(t)}_F^2 > \epsilon \lambda_r(Z(0)^TZ(0)) \right]$ are integrable by Lemma \ref{lem:leb_int_ind_contfunc}.

We can bound this integral using Lemma \ref{lem:loss-stable-ub}. This gives us
\begin{align}
    \int_{T_1}^{T_1 + T} \mathds{1}\left[\frac{d\ell(t)}{dt} > -\delta \text{ and } \ell(t) > \epsilon \right] \ dt &\leq 
    \int_{0}^{T_1+T} \mathds{1}\left[\frac{d\ell(t)}{dt} > -\delta \text{ and } \norm{R(t)}_F^2 > \epsilon \lambda_r(Z(0)^TZ(0)) \right] \ dt\\
    &\leq \frac{r \ \left( \frac{2\delta}{\epsilon \lambda_r(Z(0)^TZ(0))}-\lambda_r\left(A(0)\right)\right)^2 \ \lambda_1\left(Z(0)^TZ(0)\right)^2}{\delta}.
\end{align}
Substituting this in Equation \ref{eq:FA-final-with-delta_3}, we get
\begin{equation}
    \int_{T_1}^{T_1 + T} \frac{d\ell(t)}{dt} \ \mathds{1}\left[\frac{d\ell(t)}{dt} \ \leq \ -\delta \right] \ dt \leq -\delta \left( T \ - \ \frac{r  \left( \frac{2\delta}{\epsilon \lambda_r(Z(0)^TZ(0))}-\lambda_r\left(A(0)\right)\right)^2 \ \lambda_1\left(Z(0)^TZ(0)\right)^2}{\delta} \right).
\end{equation}
Substituting the bound for $T$, we get
\begin{equation}
    \label{eq:FA-final-with-delta_4}
    \int_{T_1}^{T_1 + T} \frac{d\ell(t)}{dt} \ \mathds{1}\left[\frac{d\ell(t)}{dt} \ \leq \ -\delta \right] \ dt \leq -\delta \left( \frac{\ell(T_1) + \ r \ \lambda_r\left(A(0)\right)^2 \ \lambda_1\left(Z(0)^TZ(0)\right)^2 - \epsilon}{\delta} \right).
\end{equation}
Combining equations \ref{eq:FA-final-with-delta_1}, \ref{eq:FA-final-with-delta_2} and \ref{eq:FA-final-with-delta_4}, we get
\begin{equation}
    \ell(T_1 + T) \leq \epsilon.
\end{equation}
\end{proof}

\begin{lemma}
\label{lem:FA-final}
For any $\epsilon > 0$ and 
\[
T \geq  \frac{24}{\epsilon} \left( \frac{\sigma_1(Y) \sigma_1(C) \sigma_1(Z(0))^4 \sqrt{r \ min(m, n)}}{\sigma_r(Z(0))^5} \right),
\]
FA dynamics satisfy 
\[
\min_{t \leq T} \ \ell(t) \leq \epsilon. 
\]
\end{lemma}
\begin{proof}
For ease of notation, let us define $\alpha, \beta, \gamma, \phi$ as follows
\begin{align*}
    \alpha &= \frac{2}{\epsilon \lambda_r(Z(0)^T Z(0))},\\
    \beta &= \lambda_r(A(0)),\\
    \gamma &= \lambda_1(Z(0)^T Z(0))^2r,\\
    \phi &= \ell(0) + r\lambda_r(A(0))^2\lambda_1(Z(0)^TZ(0))^2  - \epsilon.
\end{align*}
Then from Lemma \ref{lem:FA-final-with-delta} (with $T_1$ set to $0$), we know that for 
\[
T \geq \frac{(\alpha \delta - \beta)^2 \gamma}{\delta} + \frac{\phi}{\delta},
\]
FA dynamics satisfy 
\[
\min_{t \leq T} \ \ell(t) \leq \epsilon. 
\]
This holds for all $\delta > 0$. Minimizing the bound on $T$ with respect to $\delta$ by setting $\delta = \sqrt{\frac{\beta^2\gamma + \phi}{\gamma \alpha^2}}$, we get that for 
\begin{align}
\label{eq:FA-final_lem_0}
   T \geq 2 \alpha \gamma \sqrt{\beta^2 + \frac{\phi}{\gamma}} - 2\alpha \beta \gamma,
\end{align}
FA dynamics satisfy 
\[
\min_{t \leq T} \ \ell(t) \leq \epsilon. 
\]
In the rest of the proof, we will bound the RHS of equation \ref{eq:FA-final_lem_0}.

\begin{align*}
    2 \alpha \gamma \sqrt{\beta^2 + \frac{\phi}{\gamma}} - 2\alpha \beta \gamma &\leq  2 \alpha \gamma \sqrt{\beta^2 + \frac{\phi}{\gamma}} + 2 |\alpha \beta \gamma |\\
    &= 2 \alpha \gamma \sqrt{\beta^2 + \frac{\phi}{\gamma}} + 2 \alpha | \beta | \gamma .
\end{align*}
In the last step, we used the fact that $\alpha$ and $\gamma$ are non-negative. Substituting for $\alpha, \beta, \gamma, \phi$, we get $2 \alpha \gamma \sqrt{\beta^2 + \frac{\phi}{\gamma}} - 2\alpha \beta \gamma$
\begin{align}
\label{eq:FA-final_lem_1}
\begin{split}
     &\leq \frac{4 \ \lambda_1(Z(0)^T Z(0))^2 \ r}{\epsilon \ \lambda_r(Z(0)^T Z(0))} \left(\sqrt{\lambda_r(A(0))^2 + \frac{\ell(0) + r\lambda_r(A(0))^2\lambda_1(Z(0)^TZ(0))^2  - \epsilon}{r \ \lambda_1(Z(0)^T Z(0))^2 }} 
    + |\lambda_r(A(0))| \right)\\
&\leq \frac{4 \ \lambda_1(Z(0)^T Z(0))^2 \ r}{\epsilon \ \lambda_r(Z(0)^T Z(0))} \left(\sqrt{\lambda_r(A(0))^2 + \frac{\ell(0) + r\lambda_r(A(0))^2\lambda_1(Z(0)^TZ(0))^2  }{r \ \lambda_1(Z(0)^T Z(0))^2 }} 
    + |\lambda_r(A(0))| \right)\\
    &= \frac{4 \ \sigma_1(Z(0))^4 \ r}{\epsilon \ \sigma_r(Z(0))^2} \left(\sqrt{\lambda_r(A(0))^2 + \frac{\ell(0) + r\lambda_r(A(0))^2 \sigma_1(Z(0))^4  }{r \ \sigma_1(Z(0))^4 }} 
    + |\lambda_r(A(0))| \right)
    \end{split}
\end{align}
Now, we use the following bounds
\begin{align}
\label{eq:FA-final_lem_2}
\begin{split}
    |\lambda_r(A(0))| &= \left| \lambda_r  \left(\left(Z(0)^TZ(0)\right)^{-1}CW(0)^T + W(0)C^T\left(Z(0)^TZ(0)\right)^{-1}\right) \right|\\
    &\leq  \frac{2 \sigma_1(C) \sigma_1(W(0))}{\sigma_r(Z(0))^2}\\
    &\leq  \frac{2\sigma_1(C) \sigma_1((Y))}{\sigma_r(Z(0))^3}.
\end{split}
\end{align}
Here, for the last inequality, we use $W(0) = (Z(0)^T Z(0))^{-1} Z(0)^T Y$ and therefore $\sigma_1(W(0)) \leq \sigma_1((Y))/\sigma_r(Z(0))$.
\begin{align}
\label{eq:FA-final_lem_3}
    \begin{split}
        \ell(0) &= \norm{\left(Y - \hat{Y}(0)\right)C^T \left(Z(0)^T Z(0) \right)^{-1/2}}_F^2\\
        &\leq \frac{\sigma_1(C)^2}{\sigma_r(Z)^2} \norm{\left(Y - \hat{Y}(0)\right)}_F^2\\
        &= \leq \frac{\sigma_1(C)^2}{\sigma_r(Z)^2} \norm{\left( \ Y - \ Z(0) \left(Z(0)^T Z(0)\right)^{-1}Z(0)^T Y \ \right)}_F^2\\
        &\leq  \frac{\sigma_1(C)^2}{\sigma_r(Z)^2} \norm{Y }_F^2
    \end{split}
\end{align}
Substituting bounds in Equation \ref{eq:FA-final_lem_2} and \ref{eq:FA-final_lem_3} to Equation \ref{eq:FA-final_lem_1}, we get $2 \alpha \gamma \sqrt{\beta^2 + \frac{\phi}{\gamma}} - 2\alpha \beta \gamma$
\begin{align*}
    &\leq \frac{8 \ r \ \sigma_1(Z(0))^4 \ \sigma_1(C)  \ \sigma_1(Y)}{\epsilon \ \sigma_r(Z(0))^5} \left(\sqrt{2 + \frac{\norm{Y}_F^2 \sigma_r(Z(0))^4 }{4 r \sigma_1(Y)^2 \ \sigma_1(Z(0))^4 }} 
    + 1 \right)\\
    &\leq \frac{8 \ r \ \sigma_1(Z(0))^4 \ \sigma_1(C)  \ \sigma_1(Y)}{\epsilon \ \sigma_r(Z(0))^5} \left(\sqrt{2 + \frac{\norm{Y}_F^2  }{4 r \sigma_1(Y)^2  }} 
    + 1 \right)\\
    &\leq \frac{8 \ r \ \sigma_1(Z(0))^4 \ \sigma_1(C)  \ \sigma_1(Y)}{\epsilon \ \sigma_r(Z(0))^5} \left(\sqrt{2 + \frac{ min(m, n) \sigma_1(Y)^2  }{4 r \sigma_1(Y)^2  }}
    + 1 \right)\\
    &= \frac{8 \ r \ \sigma_1(Z(0))^4 \ \sigma_1(C)  \ \sigma_1(Y)}{\epsilon \ \sigma_r(Z(0))^5} \left(\sqrt{2 + \frac{ min(m, n)  }{4 r  }}
    + 1 \right)\\
    &\leq \frac{24}{\epsilon} \left( \frac{\sigma_1(Y) \sigma_1(C) \sigma_1(Z(0))^4 \sqrt{r \ min(m, n)}}{\sigma_r(Z(0))^5} \right)
\end{align*}
where we used $min(m, n) \geq r$ for the last inequality. Combining this with Equation \ref{eq:FA-final_lem_0}, we get that for
\[
T \geq  \frac{24}{\epsilon} \left( \frac{\sigma_1(Y) \sigma_1(C) \sigma_1(Z(0))^4 \sqrt{r \ min(m, n)}}{\sigma_r(Z(0))^5} \right),
\]
FA dynamics satisfy 
\[
\min_{t \leq T} \ \ell(t) \leq \epsilon. 
\]
\end{proof}

\begin{lemma}
\label{lem:FA-loss-final}
For any $\epsilon_1 > 0$ and 
\[
T \geq  \frac{24}{\epsilon_1} \left( \frac{\sigma_1(Y) \sigma_1(C) \sigma_1(Z(0))^6 \sqrt{r \ min(m, n)}}{\sigma_r(Z(0))^5} \right),
\]
FA dynamics satisfy 
\[
\min_{t \leq T} \ \norm{(Y - \hat{Y}(t))C^T}_F^2 \leq \epsilon_1. 
\]
\end{lemma}
\begin{proof}
\begin{align*}
\norm{(Y - \hat{Y}(t))C^T}_F^2 &= \norm{(Y - \hat{Y}(t))C^T (Z(0)^T Z(0))^{-1/2} (Z(0)^T Z(0))^{1/2}}_F^2\\
&\leq \ell(t) \ \sigma_1(Z(0))^2
\end{align*}
Applying Lemma \ref{lem:FA-final} with $\epsilon = \epsilon_1/\sigma_1(Z(0))^2$, we get that for 

\[
T \geq  \frac{24}{\epsilon_1} \left( \frac{\sigma_1(Y) \sigma_1(C) \sigma_1(Z(0))^6 \sqrt{r \ min(m, n)}}{\sigma_r(Z(0))^5} \right),
\]
FA dynamics satisfy 
\[
\min_{t \leq T} \ \ell(t) \leq \frac{\epsilon_1}{\sigma_1(Z(0))^2}, 
\]
which implies
\[
\min_{t \leq T} \ \norm{(Y-\hat{Y})C^T}_F^2 \leq \epsilon_1.
\]
\end{proof}
This finishes the proof of the first part of Theorem \ref{thm:convergence-mat-fac} (convergence of the minimum iterate).

Next, we show that $\norm{(Y-\hat{Y}(t))C^T}_F^2$ goes to $0$ as $t$ goes to $\infty$.

\begin{lemma}
\label{lem:convergence-last-iterate}
$\norm{(Y-\hat{Y}(t))C^T}_F^2 \rightarrow 0$ as $t \rightarrow \infty$.
\end{lemma}
\begin{proof}
Since $Z(t)^TZ(t) = Z(0)^T Z(0)$ (Fact \ref{fact:Z_derivative}) and since $Z(0)$ is full column rank, it is enough to show that $\ell(t) = \norm{\left(Y-\hat{Y}(t)\right)C^T \left(Z(t)^T Z(t)\right)^{-1/2}}_F^2$ goes to $0$ as $t$ goes to $\infty$.

We will show this by contradiction. Suppose $\ell(t)$ does not converge to $0$ with time. Then there must exist some $\epsilon_1 > 0$, such that for all $T \geq 0$, there exists some $t \geq T$ satisfying $\ell(t) > \epsilon_1$.

We also know from Lemma $\ref{lem:FA-final-with-delta}$, that for all $\epsilon > 0$ and for all $T \geq 0$, there exists some $t \geq T$ satisfying $l(t) \leq \epsilon$.

Using the above two arguments, we can generate an increasing infinite sequence of times $T_1, T_1', T_2, T_2', \cdots$, such that $l(T_i) \leq \epsilon_1/2$ and $l(T_i') > \epsilon_1$, for all $i \in \mathds{N}$.

By definition of this sequence,
\begin{align*}
    \int_{T_i}^{T_i'} \frac{d\ell}{dt} \ dt &= \ell(T_i') - \ell(T_i)\\
    &> \frac{\epsilon_1}{2},
\end{align*}
for all $i \in \mathds{N}$.
Let $k$ be some integer greater than $\frac{2 r \ \lambda_r\left(A(0)\right)^2 \ \lambda_1\left(Z(0)^TZ(0)\right)^2}{\epsilon_1}$. Then,
\begin{align*}
    \int_{0}^{T_k'} \frac{d\ell}{dt} \ \mathds{1}\left[ \frac{d\ell}{dt} \geq 0 \right] \ dt   &\geq  \int_{0}^{T_k'} \frac{d\ell}{dt} \  dt\\
    &\geq \sum_{i=1}^k \int_{T_i}^{T_i'} \frac{d\ell}{dt} \  dt\\
    &> \frac{k\epsilon_1}{2}\\
    &> { r \ \lambda_r\left(A(0)\right)^2 \ \lambda_1\left(Z(0)^TZ(0)\right)^2}.
\end{align*}
This is a contradiction to Lemmma \ref{lem:loss-ub}, where we show for all $T$,
\[
\int_0^T \ \frac{d  \ell(t)}{dt} \  \mathds{1}\left[\frac{d  \ell(t)}{dt} \geq 0 
\right] \ dt  \leq  r \ \lambda_r\left(A(0)\right)^2 \ \lambda_1\left(Z(0)^TZ(0)\right)^2 .
\]
Therefore, by contradiction $\ell(t)$ must converge to $0$ as $t$ goes to $\infty$.
\end{proof}
This completes the proof of Theorem \ref{thm:convergence-mat-fac}.\\

Below, we prove some helper Lemmas that we used to show that the integrals involved in the above proofs were well defined.

\begin{lemma}
\label{lem:all_bounded}
$\ell(t)$, $\frac{d\ell(t)}{dt}$, $\norm{R(t)}_F$, $\norm{A(t)}_F$ are bounded for all $t$.
\end{lemma}
\begin{proof}
The proof follows from the fact that $\norm{Z(t)}_F$ and $\norm{\left(Z(t)^TZ(t)\right)^{-1}}_F$ are bounded which holds since $Z(t)^TZ(t)$ doesn't change with time (Fact \ref{fact:Z_derivative}).
\end{proof}

\begin{lemma}
\label{lem:leb_int_ind_contfunc}
Let $f(t)$ be some continuous function of $t$ over $[0, \infty)$. Then the functions $\mathds{1}[f(t) > 0]$ and $\mathds{1}[f(t) \geq 0]$ are Lebesgue integrable over $[T_1, T_2]$ for all $T_1, T_2 \geq 0$.
\end{lemma}
\begin{proof}
  Since $f$ is a continuous function, $\{t: f(t) > 0\}$ and $\{t: f(t) \geq 0\}$ are open and closed sets respectively, which imply they are measurable. Therefore, $\mathds{1}[f(t) > 0]$ and $\mathds{1}[f(t) \geq 0]$ are bounded measurable functions. The proof follows from the fact that bounded measurable functions over any finite interval are Lebesgue integrable .
\end{proof}

\begin{lemma}
\label{lem:Rless_integrable}
$f_i(t) = \left(x^T R_{\leq k}^i(t) \right)^2$ and $g_i(t) = \norm{R_{\leq k}^i(t)}^2$ are Lebesgue integrable over $[0, T]$ for all $T$, $k$, $x$, $i$.
\end{lemma}
\begin{proof}
We can write $f_i(t) = (x^T R^i(t))^2 \ \mathds{1}[R^i(t)^T A(t) R^i(t) \leq k \norm{R^i(t)}^2]$ and \\ $g_i(t) = \norm{R^i(t)}^2 \ \mathds{1}[R^i(t)^T A(t) R^i(t) \leq k \norm{R^i(t)}^2]$.
$(x^T R^i(t))^2$, $\norm{R^i(t)}^2$ and $\mathds{1}[R^i(t)^T A(t) R^i(t) \leq k \norm{R^i(t)}^2]$  are bounded (Lemma \ref{lem:all_bounded}). Also, $(x^T R^i(t))^2$ and $\norm{R^i(t)}^2$ are integrable as they are continuous and $\mathds{1}[R^i(t)^T A(t) R^i(t) \leq k \norm{R^i(t)}^2]$ is integrable by Lemma \ref{lem:leb_int_ind_contfunc}. 
The Lemma follows since the product of bounded Lebesgue integrable functions is Lebesgue integrable (over any finite interval).
 
\end{proof}

\section{PROOF OF LEMMA \ref{lem:FA_sol_char} AND THEOREM \ref{thm:FA_over_param_opt}}
\label{sec:sol_char_proof_FA_over_param_proof}
\firstlemma*
\begin{proof}
From stationary point equation \ref{eq:stationary_FA}, we know that $(Y-\hat{Y})C^T = 0$, which gives us
\begin{align*}
    YC^T = ZWC^T
\end{align*}
This implies 
\begin{align}
\label{eq:FA_sol_char_0}
col(A) = col(YC^T) \subseteq col(Z) 
\end{align}
where $col(\cdot)$ denotes the space given the linear span of columns of the corresponding matrix. 

We also know from stationary point equation \ref{eq:stationary_FA} that $Z^T(Y - ZW) = 0$. This implies that 
\begin{align}
\label{eq:FA_sol_char_1}
W = \argmin_W \norm{ZW - Y}_F^2    
\end{align}
That is, $W$ is chosen optimally once we fix $Z$.

We consider two cases depending on rank of $Y$.\\

\textbf{Case 1}: $rank(Y) \leq r$.

Since $C_{n \times r}$ is a random matrix with i.i.d. Gaussian entries, we get $col(A) = col(Y C^T) = col(Y)$ almost surely. Thus $B$ minimizing $\norm{AB - Y}_F^2$ will satisfy $AB = Y$ almost surely.

Also, $col(Y C^T) = col(Y) \subseteq col(Z)$ almost surely (from equation \ref{eq:FA_sol_char_0}). Thus $W$ minimizing $\norm{ZW - Y}_F^2$ will satisfy $ZW = Y$ almost surely. So, we get that $\hat{Y} = AB$ almost surely.\\

\textbf{Case 2}: $rank(Y) > r$.

Since $C_{n \times r}$ is a random matrix with i.i.d. Gaussian entries and $rank(Y) > r$, $rank(YC^T) = r$ almost surely. We also know that $rank(Z_{n \times r}) \leq r = rank(YC^T) $, and $col(A) = col(YC^T) \subseteq col(Z)$ (from equation \ref{eq:FA_sol_char_0}). This implies $col(A) = col(YC^T) = col(Z)$ and $rank(A) = rank(Z) = r$ almost surely. Therefore, we can write $A = ZR$ for some invertible matrix $R$, almost surely. 
Recall that $B = \argmin_B \norm{AB - Y}_F^2$ and $W = \argmin_B \norm{AB - Y}_F^2$. Since $A$ and $Z$ are full column rank, we get $B = (A^TA)^{-1}A^TY$ and $W = (Z^TZ)^{-1}Z^TY$, almost surely. Substituting $A = ZR$, we get 
\begin{align*}
    AB &= ZR(R^T Z^T Z R)^{-1}R^T Z^T Y\\
    &= Z( Z^T Z )^{-1} Z^T Y\\
    &= ZW\\
    &= \hat{Y}
\end{align*}
almost surely. This completes the proof.
\end{proof}
\overparamopt*
\begin{proof}
In case 1 of the proof of Lemma \ref{lem:FA_sol_char}, we show ZW = Y almost surely, when $rank(Y) \leq r$. This proves Theorem \ref{thm:FA_over_param_opt}.
\end{proof}

\section{PROOF OF THEOREM \ref{thm:FA_separation_GD}}
\label{sec:FA_separation_GD_proof}
For the proofs below, we will use $P_X$ to denote the matrix projecting onto the linear span of columns of matrix $X$, and $X_{i:j}$ to denote a matrix containing column $i$ to column $j$ from matrix $X$. We use $M^{(i)}$ to denote the $i^{th}$ column of $M$.

Before proving Theorem \ref{thm:FA_separation_GD}, we prove the following helper lemma. Let $\sigma_i$, $u_i$ and $v_i$ be the $i^{\text{th}}$ singular value, left singular vector and right singular vector of $Y$ respectively, such that $Y = \sum_{i=1}^n \sigma_i u_i v_i^T = U \Sigma V^T$. 

\begin{lemma}
\label{lem:error_Y_hatY}
Let $A$ be some $n \times r$ matrix and and $P_A$ be the projection matrix for the columns space of $A$. Let $B = \argmin_B \norm{Y - AB}_F^2$. Then  $\norm{Y - AB}_F^2 = \sum_{i=1}^n \sigma_i^2 (1 - \norm{P_A u_i}_2^2)$.
\end{lemma}
\begin{proof}
Let the singular value decomposition of $A = U_A \Sigma_A V_A^T$ where $U_A$, $\Sigma_A$ and $V_A$ are $n \times rank(A)$, $rank(A) \times rank(A)$ and  $rank(A) \times r$ matrices respectively. Since $B = \argmin_B \norm{Y - AB}_F^2$, we know 

\begin{align*}
AB &=  A A^{+}Y\\
&= U_A U_A^TY
\end{align*}.
 Here $A^{+}$ denotes the pseudoinverse of $A$. This gives us 
\begin{align*}
    \norm{Y - AB}_F^2 &= \norm{Y - U_A U_A^TY}_F^2\\
    &= \norm{Y}_F^2 + Tr(Y^T U_A U_A^T U_A U_A^TY) - 2Tr(Y^T U_AU_A^TY)\\
    &=\norm{Y}_F^2 - \norm{U_A^TY}_F^2\\
    &=\sum_{i=1}^n \sigma_i^2 - \sum_{i=1}^n \sigma_i^2 \norm{U_A^T u_i}_2^2 \\
    &= \sum_{i=1}^n \sigma_i^2 - \sum_{i=1}^n \sigma_i^2 \norm{U_A U_A^T u_i}_2^2 \\
    &= \sum_{i=1}^n \sigma_i^2(1 - \norm{P_Au_i}_2^2).
\end{align*}
This completes the proof.
\end{proof}

\faseparationgd*
\begin{proof}
The part about gradient flow follows from prior results. From \citet[Theorem 39, part (b)]{bah2019learning} we know that gradient flow starting from randomly initialized $Z$ and $W$ reaches the global optimum almost surely. The global optimum here corresponds to the best rank $r$ approximation of $Y$ (in Frobenius norm) \citep{blum2020foundations} whose error $\norm{ZW-Y}_F^2$ is given by $\sum_{i=r+1}^n \sigma_i^2 = 0.5$.

Next, we lower bound the error at $Z$ and $W$ satisfying the stationary point equations for feedback alignment. From Lemma \ref{lem:FA_sol_char}, we know that $ZW = AB$ almost surely, where $A = YC^T$ and $B = \argmin_B\norm{AB-Y}_F^2$. Using Lemma \ref{lem:error_Y_hatY}, this gives us 
\begin{align}
\label{eq:thm_sep_0}
\begin{split}
    \norm{ZW-Y}_F^2 &= \norm{AB - Y}_F^2\\
    &= \sum_{i=1}^n \sigma_i^2(1 - \norm{P_Au_i}_2^2)\\
    &= 1 - \sum_{i=1}^n \sigma_i^2 \norm{P_Au_i}_2^2,
\end{split}
\end{align}
 almost surely. We used $\sum_{i=1}^n \sigma_i^2 = 1$ in the last step. 
 
 For the rest of the proof, we will assume that $A = YC^T$ is full column rank. Since $Y$ has rank $n$, and $C_{n \times r}$ is a random matrix with $r \leq n$, this is true almost surely over the choice of $C$. Since we want to bound the error with high probability (at least 0.9), we can safely ignore the cases where $A$ is not full column rank. 
 
 We want to lower bound $\norm{ZW-Y}_F^2$, for which we will upper bound $\sum_{i=1}^n \sigma_i^2 \norm{P_Au_i}_2^2$.
 Recall that $A = YC^T = \sum_{i=1}^n \sigma_i u_i v_i^TC$. We can write the $j^{\text{th}}$ column of $A$, $A^{(j)} = \sum_{i=1}^n \sigma_i u_i R_{ij}$, where $R_{ij} = \langle v_i, C^{(j)} \rangle$. 
 Since the entries of $C$ are drawn i.i.d. from $\mathcal{N}(0, 1)$ and $v_is$ are orthonormal, $R_{ij}s$ are $\mathcal{N}(0, 1)$ random variables, and are independent for all $i, j$. 
 
 Let $\hat{A}$ be the matrix obtained by applying the Gram–Schmidt orthonormalization process to the columns of $A$. That is, $\hat{A}^{(1)} = \frac{A^{(1)}}{\norm{\hat{A}^{(1)}}}$ and $\hat{A}^{(k)} = \frac{A^{(k)} - P_{A_{1:k-1}}A^{(k)}}{\norm{A^{(k)} - P_{A_{1:k-1}}A^{(k)}}}$ for $2 \leq k \leq r$. Here $P_{A_{1:k-1}}$ is the projection matrix for projecting to the space spanned by the first $k-1$ columns of $A$. We can write $P_A = \hat{A} \hat{A}^T$, which gives 
 \begin{align}
 \label{eq:thm_sep_1}
 \begin{split}
\sum_{i=1}^n \sigma_i^2 \norm{P_Au_i}_2^2  &= \sum_{i=1}^n \sigma_i^2 \norm{\hat{A}^T u_i}_2^2  \\
&= \sum_{i=1}^r \sigma_i^2 \norm{\hat{A}^T u_i}_2^2 + \sum_{i=r+1}^n \sigma_i^2 \norm{\hat{A}^T u_i}_2^2 \\
&= \frac{1}{2r}\sum_{i=1}^r \norm{\hat{A}^T u_i}_2^2 + \frac{1}{2(n-r)} \sum_{i=r+1}^n  \norm{\hat{A}^T u_i}_2^2.
\end{split}
 \end{align}
 Since $u_is$ are orthonormal, this gives us
 \begin{align*}
 \sum_{i=r+1}^n  \norm{\hat{A}^T u_i}_2^2   &\leq \norm{\hat{A}}_F^2\\
 &\leq r
 \end{align*}
 where we used the fact that $\hat{A}$ is an $n \times r$ matrix with unit length columns. Therefore we get
 \begin{align}
 \label{eq:thm_sep_2}
 \begin{split}
     \sum_{i=1}^n \sigma_i^2 \norm{P_Au_i}_2^2 &\leq \frac{1}{2r}\sum_{i=1}^r \norm{\hat{A}^T u_i}_2^2 + \frac{r}{2(n-r)} 
\end{split}     
 \end{align}
 Now we need to bound $\sum_{i=1}^r \norm{\hat{A}^T u_i}_2^2$. Let $U_{n \times n}$ be a matrix whose $j^{\text{th}}$ column $U^{(j)} = u_j$ and let $U_{i:j}$ be a matrix with columns $u_i, u_{i+1} \cdots, u_{j}$. With this notation, we need to bound $\norm{\hat{A}^T U_{1:r}}_F^2 = \sum_{j=1}^r \norm{U_{1:r}^T \hat{A}^{(j)} }_2^2$. To get a sense of this quantity, we first consider $\norm{U_{1:r}^T {A}^{(j)} }_2^2$. 
 \begin{align*}
     \norm{U_{1:r}^T {A}^{(j)} }_2^2 &= \norm{U_{1:r}^T \left(\sum_{i=1}^n \sigma_i u_i R_{ij}\right)}_2^2\\
     &= \sum_{i=1}^r \sigma_i^2 \norm{u_i}_2^2 R_{ij}^2\\
     &= \frac{1}{2r}\sum_{i=1}^r R_{ij}^2
 \end{align*}
 which is at most $\frac{1}{2}$ plus some lower order term with high probability. We will show that $\norm{U_{1:r}^T {\hat{A}}^{(j)} }_2^2$ also can not be much larger than $\frac{1}{2}$ with high probability. To show that, we consider the ratio 
 $\frac{\norm{U_{1:r}^T {\hat{A}}^{(j)} }_2^2}{\norm{U_{r+1:n}^T {\hat{A}}^{(j)} }_2^2}$. Let $P_{U_{1:r}} =U_{1:r}U_{1:r}^T $ and $P_{U_{r+1:n}} =U_{r+1:n}U_{r+1:n}^T $ be the projection matrices for the space spanned by columns of $U_{1:r}$ and $U_{r+1:n}$ respectively.  We know
 \begin{align*}
 \frac{\norm{U_{1:r}^T {\hat{A}}^{(j)} }_2^2}{\norm{U_{r+1:n}^T {\hat{A}}^{(j)} }_2^2} &= \frac{\norm{P_{U_{1:r}} {\hat{A}}^{(j)} }_2^2}{\norm{P_{U_{r+1:n}} {\hat{A}}^{(j)} }_2^2}\\
 &= \frac{\norm{P_{U_{1:r}} \left(A^{(j)} - P_{A_{1:j-1}}A^{(j)}\right) }_2^2}{\norm{P_{U_{r+1:n}} \left(A^{(j)} - P_{A_{1:j-1}}A^{(j)}\right) }_2^2}
 \end{align*}
 for $j \in \{2, 3, \cdots, r \}$ and 
  \begin{align*}
 \frac{\norm{U_{1:r}^T {\hat{A}}^{(1)} }_2^2}{\norm{U_{r+1:n}^T {\hat{A}}^{(1)} }_2^2} &= \frac{\norm{P_{U_{1:r}} {\hat{A}}^{(1)} }_2^2}{\norm{P_{U_{r+1:n}} {\hat{A}}^{(1)} }_2^2}\\
 &= \frac{\norm{P_{U_{1:r}} {{A}}^{(1)} }_2^2}{\norm{P_{U_{r+1:n}} {{A}}^{(1)} }_2^2}
 \end{align*}.
Using Lemma \ref{lem:helper-sep1} and \ref{lem:helper-sep2}, and applying a union bound, we get
\begin{align*}
 \frac{\norm{U_{1:r}^T {\hat{A}}^{(j)} }_2^2}{\norm{U_{r+1:n}^T {\hat{A}}^{(j)} }_2^2} &\leq \frac{\frac{1}{2}+ 4 \sqrt{\frac{r}{n-r}}}{\frac{1}{2} - 7 \sqrt{\frac{r}{n-r}}}
 \end{align*}
 with probability at least 0.99, for all $j \in \{1,2 \cdots, r\}$ (assuming $r\geq 2000$ and $n \geq 2r$). Since $\hat{A}^{(j)}$ is a unit vector, $\norm{U_{r+1:n}^T {\hat{A}}^{(j)} }_2^2 = 1 - {\norm{U_{1:r}^T {\hat{A}}^{(j)} }_2^2}$. This gives us
 \begin{align*}
 \frac{\norm{U_{1:r}^T {\hat{A}}^{(j)} }_2^2}{1 - {\norm{U_{1:r}^T {\hat{A}}^{(j)} }_2^2}} &\leq \frac{\frac{1}{2}+ 4 \sqrt{\frac{r}{n-r}}}{\frac{1}{2} - 7 \sqrt{\frac{r}{n-r}}}\\
 &\leq 
 \frac{\frac{1}{2}+ 7 \sqrt{\frac{r}{n-r}}}{\frac{1}{2} - 7 \sqrt{\frac{r}{n-r}}}
 \end{align*}
 with probability at least 0.99, for all $j \in \{1,2 \cdots, r\}$. Rearranging, we get
 \begin{align*}
     {\norm{U_{1:r}^T {\hat{A}}^{(j)} }_2^2} &\leq \frac{1}{2}+ 7 \sqrt{\frac{r}{n-r}}
 \end{align*}
 with probability at least 0.99, for all $j \in \{1,2 \cdots, r\}$. This gives us the desired bound on $\sum_{i=1}^r \norm{\hat{A}^T u_i}_2^2$.
 \begin{align}
 \begin{split}
     \sum_{i=1}^r \norm{\hat{A}^T u_i}_2^2 &=  \sum_{i=1}^r {\norm{U_{1:r}^T {\hat{A}}^{(j)} }_2^2}\\
     &\leq r\left( \frac{1}{2}+ 7 \sqrt{\frac{r}{n-r}} \ \right)
\end{split}     
 \end{align}
 with probability at least 0.99. Combining this with Equation \ref{eq:thm_sep_2}, we get 
 \begin{align}
 \begin{split}
     \sum_{i=1}^n \sigma_i^2 \norm{P_Au_i}_2^2 &\leq \frac{1}{2r}\sum_{i=1}^r \norm{\hat{A}^T u_i}_2^2 + \frac{r}{2(n-r)} \\
     &\leq \frac{1}{4} + \frac{7}{2}\sqrt{\frac{r}{n-r}} + \frac{r}{2(n-r)}\\
     &\leq \frac{1}{4} + 4\sqrt{\frac{r}{n-r}}\\
     &\leq \frac{1}{4} + 0.01
\end{split}
 \end{align}
 with probability at least 0.99 (assume $n \geq 40001 \ r$). Combining this with equation \ref{eq:thm_sep_0}, we get
 \begin{align}
     \begin{split}
         \norm{ZW-Y}_F^2 &= 1 - \sum_{i=1}^n \sigma_i^2 \norm{P_Au_i}_2^2\\
         &\geq 0.74
     \end{split}
 \end{align}
 with probability at least 0.99, assuming $2000 \leq r \leq \frac{1}{40001} n$.  
\end{proof}

We use Lemma \ref{lem:helper-sep1}, \ref{lem:helper-sep2} and \ref{lem:helper-sep3} for proving Theorem \ref{thm:FA_separation_GD}. These Lemmas are proved assuming the conditions (i) to (iv) in Theorem \ref{thm:FA_separation_GD} are satisfied. 

\begin{lemma}
\label{lem:helper-sep1}
Let $r \geq 2000$ and $n \geq 2r$. With probability at least 0.997,
\[
 \norm{P_{U_{1:r}} \left(A^{(j)} - P_{A_{1:j-1}}A^{(j)}\right) }_2^2 \leq \frac{1}{2} +  4\sqrt{\frac{r}{n-r}},
 \]
 \[
 \norm{P_{U_{1:r}} \left(A^{(1)} \right) }_2^2 \leq \frac{1}{2} +  4\sqrt{\frac{r}{n-r}},
 \]
for all $j \in \{ 2, 3 \cdots, r\}$.
\end{lemma}
\begin{proof}
For $2 \leq j \leq r$)
 \begin{align}
 \label{eq:lem-helper-sep1_1}
 \begin{split}
     \norm{P_{U_{1:r}} \left(A^{(j)} - P_{A_{1:j-1}}A^{(j)}\right) }_2^2
     &= \norm{P_{U_{1:r}} A^{(j)} - P_{U_{1:r}} P_{A_{1:j-1}}A^{(j)} }_2^2 \\
     &= \norm{P_{U_{1:r}} A^{(j)} - P_{U_{1:r}} P_{A_{1:j-1}}\left(P_{U_{1:r}} + P_{U_{r+1:n}} \right)A^{(j)} }_2^2\\
     &= \norm{P_{U_{1:r}} A^{(j)} - P_{U_{1:r}} P_{A_{1:j-1}}P_{U_{1:r}}A^{(j)} - P_{U_{1:r}} P_{A_{1:j-1}}P_{U_{r+1:n}}A^{(j)}}_2^2\\
     &\leq \norm{P_{U_{1:r}} A^{(j)} - P_{U_{1:r}} P_{A_{1:j-1}}P_{U_{1:r}}A^{(j)} }_2^2 + \norm{P_{U_{1:r}} P_{A_{1:j-1}}P_{U_{r+1:n}}A^{(j)}}_2^2\\  &+2\norm{P_{U_{1:r}} A^{(j)} - P_{U_{1:r}} P_{A_{1:j-1}}P_{U_{1:r}}A^{(j)} }_2 \norm{P_{U_{1:r}} P_{A_{1:j-1}}P_{U_{r+1:n}}A^{(j)}}_2
     \end{split}
 \end{align}
 Now we bound the individual terms in the last inequality. 
 \begin{align*}
         \norm{P_{U_{1:r}} A^{(j)} - P_{U_{1:r}} P_{A_{1:j-1}}P_{U_{1:r}}A^{(j)} }_2^2 &= \norm{P_{U_{1:r}} A^{(j)}}_2^2 + \norm{ P_{U_{1:r}} P_{A_{1:j-1}}P_{U_{1:r}}A^{(j)} }_2^2 \\
         & - 2 \langle P_{U_{1:r}} A^{(j)}, P_{U_{1:r}} P_{A_{1:j-1}}P_{U_{1:r}}A^{(j)}  \rangle
 \end{align*}
 Using $\norm{ P_{U_{1:r}} P_{A_{1:j-1}}P_{U_{1:r}}A^{(j)} }_2^2 \leq \norm{  P_{A_{1:j-1}}P_{U_{1:r}}A^{(j)} }_2^2$ and\\
 $\langle P_{U_{1:r}} A^{(j)}, P_{U_{1:r}} P_{A_{1:j-1}}P_{U_{1:r}}A^{(j)} \rangle = \norm{  P_{A_{1:j-1}}P_{U_{1:r}}A^{(j)} }_2^2$, we get 
 \begin{align*}
      \norm{P_{U_{1:r}} A^{(j)} - P_{U_{1:r}} P_{A_{1:j-1}}P_{U_{1:r}}A^{(j)} }_2^2 &\leq \norm{P_{U_{1:r}} A^{(j)}}_2^2 
 \end{align*}
 Substituting $A^{(j)} = \sum_{i=1}^n \sigma_i u_i R_{ij}$ and $P_{{U_{1:r}}} = {U_{1:r}}{U_{1:r}}^T$, we get 
  \begin{align*}
      \norm{P_{U_{1:r}} A^{(j)} - P_{U_{1:r}} P_{A_{1:j-1}}P_{U_{1:r}}A^{(j)} }_2^2 &\leq \sum_{i=1}^r \sigma_i^2 R_{ij}^2\\
      &= \frac{1}{2r} \sum_{i=1}^r R_{ij}^2
 \end{align*}
 Using concentration (Lemma \ref{lem:concentration-chisq}) and a union bound, we  know 
 \begin{align}
     Pr\left[\sum_{i=1}^r R_{ij}^2 \geq r + 16 \sqrt{r \ log(r)}  \right] \leq \frac{2}{r} &\leq \frac{1}{1000}
 \end{align}
 for all $j \in \{1, 2, \cdots, r\}$, for $r \geq 2000$.   Therefore we get 
 \begin{align}
 \label{eq:lem-helper-sep1_3}
      \norm{P_{U_{1:r}} A^{(j)} - P_{U_{1:r}} P_{A_{1:j-1}}P_{U_{1:r}}A^{(j)} }_2^2 &\leq \frac{1}{2} + \frac{8 \sqrt{log(r)}}{\sqrt{r}}
 \end{align}
 with probability at least $0.999$ for all $j \in \{ 2,3 \cdots, r\}$. 
 Next, we bound the $\norm{P_{U_{1:r}} P_{A_{1:j-1}}P_{U_{r+1:n}}A^{(j)}}_2^2$ term in Equation \ref{eq:lem-helper-sep1_1}. 
\begin{align}
\label{eq:lem-helper-sep1_2}
\begin{split}
    \norm{P_{U_{1:r}} P_{A_{1:j-1}}P_{U_{r+1:n}}A^{(j)}}_2^2 &\leq \norm{ P_{A_{1:j-1}}P_{U_{r+1:n}}A^{(j)}}_2^2\\
    &\leq \frac{r}{n-r}
\end{split}
\end{align}
with probability at least 0.999, for all $j \in \{2,3, \cdots, n \}$ (assuming $r \geq 2000$, $n \geq 2r$ ). We prove the last inequality in Lemma \ref{lem:helper-sep3}.
 
 Combining Equations \ref{eq:lem-helper-sep1_1}, \ref{eq:lem-helper-sep1_3} and \ref{eq:lem-helper-sep1_2}, and applying a union bound, we get 
 \begin{align}
 \label{eq:lem-helper-sep1_7}
     \begin{split}
         \norm{P_{U_{1:r}} \left(A^{(j)} - P_{A_{1:j-1}}A^{(j)}\right) }_2^2 &\leq \frac{1}{2} + \frac{8 \sqrt{log(r)}}{\sqrt{r}} + \frac{r}{n-r} + 2 \sqrt{\frac{r}{n-r}}\sqrt{\frac{1}{2} + \frac{8 \sqrt{log(r)}}{\sqrt{r}}}\\
         &\leq \frac{1}{2} + \frac{8 \sqrt{log(r)}}{\sqrt{r}} + 3\sqrt{\frac{r}{n-r}}\\
         &\leq \frac{1}{2} + 4\sqrt{\frac{r}{n-r}}
     \end{split}
 \end{align}
 with probability at least 0.998, for all $j \in \{2,3, \cdots, r\}$ (assuming $r \geq 2000$ and $n \geq 2r$). For the $j=1$ case, we know
 \begin{align}
 \label{eq:lem-helper-sep1_8}
     \begin{split}
          \norm{P_{U_{1:r}} {{A}}^{(1)} }_2^2
         &= \sum_{i=1}^r \sigma_{i1}^2 R_{i1}^2\\
         & = \frac{1}{2r}\sum_{i=1}^r{R_{i1}}^2\\
         &\leq \frac{1}{2} + \frac{8 \sqrt{log(r)}}{\sqrt{r}}\\
         &\leq \frac{1}{2} + 4\sqrt{\frac{r}{n-r}}
     \end{split}
 \end{align}
 with probability at least 0.999 (for $r \geq 2000$). Here we used the concentration inequality from Lemma \ref{lem:concentration-chisq}. Combining Equations \ref{eq:lem-helper-sep1_7} and \ref{eq:lem-helper-sep1_8}, and applying a union bound, we get 
\[
 \norm{P_{U_{1:r}} \left(A^{(j)} - P_{A_{1:j-1}}A^{(j)}\right) }_2^2 \leq \frac{1}{2} + 4\sqrt{\frac{r}{n-r}}
 \]
 \[
 \norm{P_{U_{1:r}} \left(A^{(1)} \right) }_2^2 \leq \frac{1}{2} + 4\sqrt{\frac{r}{n-r}}
 \]
 with probability at least 0.997, for all $j \in \{2, 3 \cdots, r\}$ (assuming $r \geq 2000$ and $n \geq 2r$).
\end{proof}

\begin{lemma}
\label{lem:helper-sep2}
Let $r \geq 2000$ and $n \geq 2r$. With probability at least 0.993,
\[
 \norm{P_{U_{r+1:n}} \left(A^{(j)} - P_{A_{1:j-1}}A^{(j)}\right) }_2^2 \geq \frac{1}{2} - \frac{7 \sqrt{r}}{\sqrt{n-r}},
 \]
 \[
 \norm{P_{U_{r+1:n}} \left(A^{(1)} \right) }_2^2 \geq \frac{1}{2} - \frac{7 \sqrt{r}}{\sqrt{n-r}}
 \]
for all $j \in \{ 2, 3 \cdots, r\}$.
\end{lemma}
\begin{proof}
For $j \in \{2,3 \cdots, r\}$,
 \begin{align}
 \label{eq:lem-helper-sep2_1}
 \begin{split}
     \norm{P_{U_{r+1:n}} \left(A^{(j)} - P_{A_{1:j-1}}A^{(j)}\right) }_2^2
     &= \norm{P_{U_{r+1:n}} A^{(j)} - P_{U_{r+1:n}} P_{A_{1:j-1}}A^{(j)} }_2^2 \\
     &= \norm{P_{U_{r+1:n}} A^{(j)} - P_{U_{r+1:n}} P_{A_{1:j-1}}\left(P_{U_{1:r}} + P_{U_{r+1:n}} \right)A^{(j)} }_2^2\\
     &= \norm{P_{U_{r+1:n}} A^{(j)} - P_{U_{r+1:n}} P_{A_{1:j-1}}P_{U_{1:r}}A^{(j)} - P_{U_{r+1:n}} P_{A_{1:j-1}}P_{U_{r+1:n}}A^{(j)}}_2^2\\
     &\geq \norm{P_{U_{r+1:n}} A^{(j)} - P_{U_{r+1:n}} P_{A_{1:j-1}}P_{U_{1:r}}A^{(j)} }_2^2 + \norm{P_{U_{r+1:n}} P_{A_{1:j-1}}P_{U_{r+1:n}}A^{(j)}}_2^2\\  &-2\norm{P_{U_{r+1:n}} A^{(j)} - P_{U_{r+1:n}} P_{A_{1:j-1}}P_{U_{1:r}}A^{(j)} }_2  \norm{P_{U_{r+1:n}} P_{A_{1:j-1}}P_{U_{r+1:n}}A^{(j)}}_2\\
     &\geq \norm{P_{U_{r+1:n}} A^{(j)} - P_{U_{r+1:n}} P_{A_{1:j-1}}P_{U_{1:r}}A^{(j)} }_2^2 \\  &-2\norm{P_{U_{r+1:n}} A^{(j)} - P_{U_{r+1:n}} P_{A_{1:j-1}}P_{U_{1:r}}A^{(j)} }_2  \norm{P_{U_{r+1:n}} P_{A_{1:j-1}}P_{U_{r+1:n}}A^{(j)}}_2
     \end{split}
 \end{align}
 Now we bound the individual terms in the last inequality. We first lower bound the $\norm{P_{U_{r+1:n}} A^{(j)} - P_{U_{r+1:n}} P_{A_{1:j-1}}P_{U_{1:r}}A^{(j)} }_2^2$ term.
 \begin{align}
 \label{eq:lem-helper-sep2_2}
     \norm{P_{U_{r+1:n}} A^{(j)} - P_{U_{r+1:n}} P_{A_{1:j-1}}P_{U_{1:r}}A^{(j)} }_2^2 \geq 
      \norm{P_{U_{r+1:n}} A^{(j)}  }_2^2 - 2 \langle P_{U_{r+1:n}} A^{(j)}, P_{U_{r+1:n}} P_{A_{1:j-1}}P_{U_{1:r}}A^{(j)} \rangle
 \end{align}
 We lower bound the $\norm{P_{U_{r+1:n}} A^{(j)}  }_2^2$ term:
 \begin{align}
 \label{eq:lem-helper-sep2_3}
 \begin{split}
     \norm{P_{U_{r+1:n}} A^{(j)}  }_2^2 &= \sum_{i=r+1}^n \sigma_i^2 R_{ij}^2\\
     &= \frac{1}{2(n-r)} \sum_{i=r+1}^n  R_{ij}^2\\
     &\geq \frac{1}{2} - \frac{8 \sqrt{log(r)}}{\sqrt{n-r}}
\end{split}
 \end{align}
 with probability at least 0.999, for all $j \in \{1, 2 \cdots, r \}$. Here we used concentration of chi-squared random variables (Lemma \ref{lem:concentration-chisq}) and a union bound. We also assumed $r \geq 2000$ and $n \geq 2r$.
 
 Next, we upper bound $\langle P_{U_{r+1:n}} A^{(j)}, P_{U_{r+1:n}} P_{A_{1:j-1}}P_{U_{1:r}}A^{(j)} \rangle$.
 \begin{align}
  \label{eq:lem-helper-sep2_4}
 \begin{split}
     \langle P_{U_{r+1:n}} A^{(j)}, P_{U_{r+1:n}} P_{A_{1:j-1}}P_{U_{1:r}}A^{(j)} \rangle &= \langle P_{U_{r+1:n}} A^{(j)}, \  P_{A_{1:j-1}}P_{U_{1:r}}A^{(j)} \rangle\\
     &= \left\langle \sum_{i=r+1}^n u_i \sigma_i R_{ij}, \ P_{A_{1:j-1}} \left( \sum_{i=1}^r u_i \sigma_i R_{ij}  \right) \right\rangle\\
     &= \frac{\norm{P_{A_{1:j-1}} \left( \sum_{i=1}^r u_i  R_{ij}  \right)}_2}{2\sqrt{r(n-r)}} \left\langle \sum_{i=r+1}^n u_i  R_{ij}, \ \frac{P_{A_{1:j-1}} \left( \sum_{i=1}^r u_i  R_{ij}  \right)}{\norm{P_{A_{1:j-1}} \left( \sum_{i=1}^r u_i  R_{ij}  \right)}_2} \right\rangle
\end{split}
 \end{align}
 Now, note that $\sum_{i=r+1}^n u_i  R_{ij}$ and $P_{A_{1:j-1}} \left( \sum_{i=1}^r u_i  R_{ij}  \right)$ are independent since the latter only depends on $R_{ij}$ for $i \leq r$ and $R_{ik}$ for $k \leq j-1$.  
 Also, for any fixed $P_{A_{1:j-1}} \left( \sum_{i=1}^r u_i  R_{ij}  \right)$, there exists a unit norm vector 
 $x$  lying in the linear span of $\{u_{r+1}, u_{r+2}, \cdots, u_{n}\}$ , such that  
 \begin{align}
\left| \left\langle \sum_{i=r+1}^n u_i  R_{ij}, \ \frac{P_{A_{1:j-1}} \left( \sum_{i=1}^r u_i  R_{ij}  \right)}{\norm{P_{A_{1:j-1}} \left( \sum_{i=1}^r u_i  R_{ij}  \right)}_2} \right\rangle \right| &\leq \left| \left\langle \sum_{i=r+1}^n u_i  R_{ij}, x \right\rangle \right|
 \end{align}
 for all values of $R_{ij}$s. Therefore, we can write
 \begin{align}
 \label{eq:lem-helper-sep2_5}
     Pr\left[ \left| \left\langle \sum_{i=r+1}^n u_i  R_{ij}, \ \frac{P_{A_{1:j-1}} \left( \sum_{i=1}^r u_i  R_{ij}  \right)}{\norm{P_{A_{1:j-1}} \left( \sum_{i=1}^r u_i  R_{ij}  \right)}_2} \right\rangle \right| \geq \alpha \right] &\leq Pr\left[ \left| \left\langle \sum_{i=r+1}^n u_i  R_{ij}, x \right\rangle \right| \geq \alpha \right]
 \end{align}
 for all $\alpha>0$, where the probability is over the randomness of $\sum_{i=r+1}^n u_i  R_{ij}$ and $P_{A_{1:j-1}} \left( \sum_{i=1}^r u_i  R_{ij}  \right)$ is fixed. Now, since $\sum_{i=r+1}^n  u_i R_{ij}$ is an isotropic Gaussian random variable in the space spanned by $\{u_{r+1}, u_{r+2}, \cdots, u_{n}\}$, we know that $\left\langle \sum_{i=r+1}^n u_i  R_{ij}, x \right\rangle$ and $\left\langle \sum_{i=r+1}^n u_i  R_{ij},  u_{r+1} \right\rangle$ are equal in distribution. Here we used the fact that $x$  is a unit vector lying in the span  of $\{u_{r+1}, u_{r+2}, \cdots, u_{n}\}$. Therefore 
 \begin{align*}
     Pr\left[ \left| \left\langle \sum_{i=r+1}^n u_i  R_{ij}, x \right\rangle \right| \geq \alpha  \right] &= Pr\left[ \left| \left\langle \sum_{i=r+1}^n u_i  R_{ij},  u_{r+1} \right\rangle \right| \geq \alpha  \right]\\
     &= Pr\left[ | R_{r+1 \ j} | \geq \alpha \right].
 \end{align*}
 Combining this with Equation \ref{eq:lem-helper-sep2_5}, and setting $\alpha = 2 \sqrt{ log(r)}$, we get 
 \begin{align*}
     Pr\left[ \left| \left\langle \sum_{i=r+1}^n u_i  R_{ij}, \ \frac{P_{A_{1:j-1}} \left( \sum_{i=1}^r u_i  R_{ij}  \right)}{\norm{P_{A_{1:j-1}} \left( \sum_{i=1}^r u_i  R_{ij}  \right)}_2} \right\rangle \right| \geq 2 \sqrt{ log(r)} \right] &\leq Pr\left[|R_{r+1 \ j}| \geq 2 \sqrt{log(r)}\right]
 \end{align*}
 Using a standard tail bound for Normal random variables \citep{wainwright2015basic},  and a union bound, we  get
 \begin{align}
 \label{eq:lem-helper-sep2_6}
     \left| \left\langle \sum_{i=r+1}^n u_i  R_{ij}, \ \frac{P_{A_{1:j-1}} \left( \sum_{i=1}^r u_i  R_{ij}  \right)}{\norm{P_{A_{1:j-1}} \left( \sum_{i=1}^r u_i  R_{ij}  \right)}_2} \right\rangle \right| &\leq 2 \sqrt{log(r)}
 \end{align}
 with probability at least 0.999 for all $j \in \{2, 3,  \cdots, r\}$ where the probability is over $\sum_{i=r+1}^n u_i  R_{ij}$. But since this holds for any fixed $P_{A_{1:j-1}} \left( \sum_{i=1}^r u_i  R_{ij}  \right)$, this  holds even when we take probability over both $\sum_{i=r+1}^n u_i  R_{ij}$ and $P_{A_{1:j-1}} \left( \sum_{i=1}^r u_i  R_{ij}  \right)$. Also,
 \begin{align}
 \label{eq:lem-helper-sep2_7}
 \begin{split}
     {\norm{P_{A_{1:j-1}} \left( \sum_{i=1}^r u_i  R_{ij}  \right)}_2}^2 &\leq {\norm{  \sum_{i=1}^r u_i  R_{ij} }_2}^2\\
     &= \sum_{i=1}^r R_{ij}^2\\
     &\leq r + 16\sqrt{r \ log(r)}
\end{split}
 \end{align}
 with probability at least $0.999$, for all $j \in \{2,3 \cdots, r\}$ (assuming $r \geq 2000$). Here we used  concentration (Lemma \ref{lem:concentration-chisq}) and a union bound.
 Combining Equations \ref{eq:lem-helper-sep2_4}, \ref{eq:lem-helper-sep2_6} and \ref{eq:lem-helper-sep2_7}, and using a union bound, we get
 \begin{align}
 \label{eq:lem-helper-sep2_8}
     \begin{split}
     \langle P_{U_{r+1:n}} A^{(j)}, P_{U_{r+1:n}} P_{A_{1:j-1}}P_{U_{1:r}}A^{(j)} \rangle &\leq \sqrt{\frac{\left({r + 16 \sqrt{r \ log(r)}}\right) \ 4 \ log(r)}{2r(n-r)}}\\
     &\leq 2 \sqrt{\frac{log(r)}{n-r}}
     \end{split}
 \end{align}
 with probability at least 0.998, for all $j \in \{2,3,\cdots, r \} $, assuming $r \geq 2000$.
 Combining equations \ref{eq:lem-helper-sep2_2}, \ref{eq:lem-helper-sep2_3} and \ref{eq:lem-helper-sep2_8}, and applying a union bound, we get 
 \begin{align}
 \begin{split}
  \label{eq:lem-helper-sep2_9}
     \norm{P_{U_{r+1:n}} A^{(j)} - P_{U_{r+1:n}} P_{A_{1:j-1}}P_{U_{1:r}}A^{(j)} }_2^2 &\geq \frac{1}{2} - \frac{8 \sqrt{log(r)}}{\sqrt{n-r}} -\frac{4 \sqrt{log(r)}}{\sqrt{n-r}}\\
     &= \frac{1}{2} - \frac{12 \sqrt{log(r)}}{\sqrt{n-r}}
\end{split}     
 \end{align}
 with probability at least 0.997, for all $j \in \{2, 3 \cdots, r \}$.
 We also upper bound $\norm{P_{U_{r+1:n}} A^{(j)} - P_{U_{r+1:n}} P_{A_{1:j-1}}P_{U_{1:r}}A^{(j)} }_2^2$.
 \begin{align}
 \label{eq:lem-helper-sep2_10}
     \begin{split}
      \norm{P_{U_{r+1:n}} A^{(j)} - P_{U_{r+1:n}} P_{A_{1:j-1}}P_{U_{1:r}}A^{(j)} }_2^2 &\leq \norm{P_{U_{r+1:n}} A^{(j)}}^2 + \norm{ P_{U_{r+1:n}} P_{A_{1:j-1}}P_{U_{1:r}}A^{(j)} }^2\\
      &+ 2 \norm{P_{U_{r+1:n}} A^{(j)}}\norm{ P_{U_{r+1:n}} P_{A_{1:j-1}}P_{U_{1:r}}A^{(j)} }\\
      &\leq 4 \norm{A^{(j)}}^2\\
      &= 4 \sum_{i=1}^n \sigma_i^2 R_{ij}^2\\
      &= \frac{2}{r}\sum_{i=1}^r R_{ij}^2 + \frac{2}{n-r} \sum_{i=r+1}^n R_{ij}^2\\
      &\leq 4+\frac{32\sqrt{log(r)}}{\sqrt{r}} + \frac{32\sqrt{log(r)}}{\sqrt{n-r}}\\
      &\leq 8
     \end{split}
 \end{align}
 with probability at least 0.998 for all $j \in {2,3, \cdots, r}$ (assuming $r \geq 2000$ and $n \geq 2r$). Here we applied concentration (Lemma \ref{lem:concentration-chisq}) and a union bound.
 
 We also need to upper bound $\norm{P_{U_{r+1:n}} P_{A_{1:j-1}}P_{U_{r+1:n}}A^{(j)}}_2$.
 \begin{align}
\label{eq:lem-helper-sep2_11}
\begin{split}
    \norm{P_{U_{r+1:n}} P_{A_{1:j-1}}P_{U_{r+1:n}}A^{(j)}}_2^2 &\leq \norm{ P_{A_{1:j-1}}P_{U_{r+1:n}}A^{(j)}}_2^2\\
    &\leq \frac{r}{n-r}
\end{split}
\end{align}
with probability at least 0.999, for all $j \in \{2,3, \cdots, n \}$ (assuming $r \geq 2000$, $n \geq 2r$ ). We prove the last inequality in Lemma \ref{lem:helper-sep3}.

Combining Equations \ref{eq:lem-helper-sep2_1}, \ref{eq:lem-helper-sep2_9}, \ref{eq:lem-helper-sep2_10} and \ref{eq:lem-helper-sep2_11}, and applying a union bound, we get
\begin{align}
\label{eq:lem-helper-sep2_12}
    \begin{split}
    \norm{P_{U_{r+1:n}} \left(A^{(j)} - P_{A_{1:j-1}}A^{(j)}\right) }_2^2 &\geq \frac{1}{2} - \frac{12 \sqrt{log(r)}}{\sqrt{n-r}} - 2\sqrt{8}\sqrt{\frac{r}{n-r}}\\
    &\geq \frac{1}{2} - \frac{7 \sqrt{r}}{\sqrt{n-r}}
    \end{split}
\end{align}
with probability at least 0.994 for all $j \in \{2,3, \cdots, r\}$ (assuming $r \geq 2000$, $n\geq 2r$).

To prove the lemma, we also need to bound $\norm{P_{U_{r+1:n}} A^{(1)} }_2^2$.
\begin{align}
\label{eq:lem-helper-sep2_13}
    \begin{split}
        \norm{P_{U_{r+1:n}} A^{(1)} }_2^2 &= \sum_{i=r+1}^n \sigma_i^2 R_{i1}^2\\
        &= \frac{1}{2(n-r)}\sum_{i=r+1}^n R_{i1}^2\\
        &\geq \frac{1}{2} - \frac{16\sqrt{log(r)}}{\sqrt{n-r}}\\
        &\geq \frac{1}{2} - \frac{7 \sqrt{r}}{\sqrt{n-r}}
    \end{split}
\end{align}
with probability at least 0.999 (assuming $r \geq 2000$, $n \geq 2r$). Here we used the concentration inequality from Lemma \ref{lem:concentration-chisq}.

Applying a union bound, we get that both Equations \ref{eq:lem-helper-sep2_12} and \ref{eq:lem-helper-sep2_13} hold with probability at least 0.993 for all $j \in \{2, 3, \cdots, r\}$. This completes the proof.
\end{proof}
\begin{lemma}
\label{lem:helper-sep3}
Let $r \geq 2000$ and $n \geq 2r$. Then
\[
\norm{ P_{A_{1:j-1}}P_{U_{r+1:n}}A^{(j)}}_2^2 
 \leq \frac{r}{n-r}.
\]
with probability at least 0.999 for all $j \in \{2,3, \cdots, r\}$.
\end{lemma}
\begin{proof}
\begin{align}
\label{eq:lem-helper-sep3_1}
\begin{split}
     \norm{ P_{A_{1:j-1}}P_{U_{r+1:n}}A^{(j)}}_2^2
    &\leq \norm{ P_{A_{1:j-1}} \left( \sum_{i=r+1}^n \sigma_i u_i R_{ij} \right)}_2^2\\
    &= \frac{1}{2(n-r)} \norm{ P_{A_{1:j-1}} \left( \sum_{i=r+1}^n  u_i R_{ij} \right)}_2^2
\end{split}
\end{align}
 Note that $P_{A_{1:j-1}}$ and $\left( \sum_{i=r+1}^n  u_i R_{ij} \right)$  are independent since  $P_{A_{1:j-1}}$ only depends on random variables $R_{ik}$ for $k \leq j-1$. Also, for any fixed $P_{A_{1:j-1}}$, there exists a projection matrix $P_{j-1}$ that project onto some $j-1$ dimensional subspace of the linear span of $\{u_{r+1}, u_{r+2}, \cdots, u_{n}\}$ (assume $n \geq 2r$), such that
 \begin{align*}
     \norm{ P_{A_{1:j-1}} \left( \sum_{i=r+1}^n  u_i R_{ij} \right)}_2^2 &\leq \norm{ P_{{j-1}} \left( \sum_{i=r+1}^n  u_i R_{ij} \right)}_2^2
 \end{align*}
 for all values of $R_{ij}$s. Therefore, we can write
 \begin{align}
 \label{eq:lem-helper-sep1_4}
     Pr\left[ \norm{ P_{A_{1:j-1}} \left( \sum_{i=r+1}^n  u_i R_{ij} \right)}_2^2 \geq \alpha \right] &\leq Pr\left[ \norm{ P_{{j-1}} \left( \sum_{i=r+1}^n  u_i R_{ij}  \right)}_2^2 \geq \alpha \right]
 \end{align}
 for all $\alpha>0$, where the probability is over the randomness of $R_{ij}$s and $ P_{A_{1:j-1}}$ is fixed. Now, since $\sum_{i=r+1}^n  u_i R_{ij}$ is an isotropic Gaussian random variable in the space spanned by $\{u_{r+1}, u_{r+2}, \cdots, u_{n}\}$, we know that $\norm{ P_{{j-1}} \left( \sum_{i=r+1}^n  u_i R_{ij} \right)}_2^2$ and $\norm{ P_{U_{r+1:r+j-1}} \left( \sum_{i=r+1}^n  u_i R_{ij} \right)}_2^2$ are equal in distribution. Here we used the fact that $P_{j-1}$  projects onto some $j-1$ dimensional subspace of the linear span of $\{u_{r+1}, u_{r+2}, \cdots, u_{n}\}$. Therefore 
 \begin{align*}
     Pr\left[ \norm{ P_{{j-1}} \left( \sum_{i=r+1}^n  u_i R_{ij}  \right)}_2^2 \geq \alpha \right] &= Pr\left[ \norm{ P_{U_{r+1:r+j-1}} \left( \sum_{i=r+1}^n  u_i R_{ij}  \right)}_2^2 \geq \alpha \right]\\
     &= Pr\left[ \sum_{i=r+1}^{r+j-1} R_{ij}^2 \geq \alpha \right]\\
     &\leq Pr\left[ \sum_{i=r+1}^{2r} R_{ij}^2 \geq \alpha \right]
 \end{align*}
 for $j \leq r$.
 Combining this with Equation \ref{eq:lem-helper-sep1_4}, and setting $\alpha = r + 16 \sqrt{r \ log(r)}$, we get 
 \begin{align}
     Pr\left[ \norm{ P_{A_{1:j-1}} \left( \sum_{i=r+1}^n  u_i R_{ij} \right)}_2^2 \geq r + 16 \sqrt{r \ log(r)} \right] \leq Pr\left[ \sum_{i=r+1}^{2r} R_{ij}^2 \geq r + 16 \sqrt{r \ log(r)} \right].
 \end{align}
 Using concentration (Lemma \ref{lem:concentration-chisq}) and a union bound, we  know 
 \begin{align*}
     Pr\left[\sum_{i=r+1}^{2r} R_{ij}^2 \geq r + 16 \sqrt{r \ log(r)}  \right] \leq \frac{2}{r} &\leq \frac{1}{1000}
 \end{align*}
 for all $j \in \{1, 2, \cdots, r \}$. Therefore, we get
 \begin{align}
     \norm{ P_{A_{1:j-1}} \left( \sum_{i=r+1}^n  u_i R_{ij} \right)}_2^2 \leq r + 16 \sqrt{r \ log(r)}
 \end{align}
 with probability at least 0.999 for all $j \in \{2, 3,  \cdots, r\}$ where the probability is over $R_{ij}$s. But since this holds for any fixed $P_{A_{1:j-1}}$, this  holds even when we take probability over both $R_{ij}$s and $P_{A_{1:j-1}}$. Substituting this in Equation \ref{eq:lem-helper-sep3_1}, we get that for $r \geq 2000$,
 \begin{align}
 \label{eq:lem-helper-sep1_5}
 \begin{split}
 \norm{ P_{A_{1:j-1}}P_{U_{r+1:n}}A^{(j)}}_2^2 &\leq    \frac{r + 16 \sqrt{r \ log(r)}}{2(n-r)} \\
 &\leq \frac{r}{n-r}.
 \end{split}
 \end{align}
 with probability at least 0.999 for all $j \in \{2,3,  \cdots, r\}$.
\end{proof}
\begin{lemma}
\label{lem:concentration-chisq}
\citep{wainwright2015basic}
Let $X$ be a $\chi^2$ random variable with $d$ degrees of freedom, then
\[
Pr[|X - d| \geq dt] \leq 2e^{-dt^2/8}
\]
for all $t \in (0, 1)$.
\end{lemma}
\section{PROOF OF THEOREM \ref{thm:FA_GD_rep_orthogonal}}
\label{sec:proof_FA_GD_rep_orthogonal}
\fagdorthrep*
\begin{proof}
Let $\sigma_i$, $u_i$ and $v_i$ be the $i^{\text{th}}$ singular value, left singular vector and right singular vector of $Y$ respectively, such that $Y = \sum_{i=1}^n \sigma_i u_i v_i^T = U \Sigma V^T$. 

From \citet[Theorem 39, part (b)]{bah2019learning} we know that gradient flow starting from randomly initialized $Z$ and $W$ reaches the global optimum almost surely. The global optimum here corresponds to the best rank $1$ approximation of $Y$ (in frobenious norm) \citep{blum2020foundations} whose error $\norm{ZW-Y}_F^2$ is given by $\sum_{i=2}^n \sigma_i^2 = \epsilon^2(n-1)$. So, from prior work, we know that
\begin{align}
    \frac{Z_{GD}}{\norm{Z_{GD}}_2} = u_1
\end{align}
and
\begin{align}
\label{eq:thm_rep_orth_1}
\norm{Z_{GD}W_{GD}-Y}_F^2 = \epsilon^2(n-1)
\end{align}
almost surely.

Next, we consider $Z_{FA}$ and $W_{FA}$ satisfying the stationary point equations for feedback alignment. From Lemma \ref{lem:FA_sol_char}, we know that $Z_{FA}W_{FA} = AB$ almost surely, where $A = YC^T $ and $B = \argmin_B\norm{AB-Y}_F^2$.

Using Lemma \ref{lem:error_Y_hatY}, this gives us 
\begin{align}
\label{eq:thm_rep_orth_2}
\begin{split}
    \norm{Z_{FA}W_{FA}-Y}_F^2 &= \norm{AB - Y}_F^2\\
    &= \sum_{i=1}^n \sigma_i^2(1 - \norm{P_Au_i}_2^2)\\
    &= \epsilon^2(n-1)+1 - \sum_{i=1}^n \sigma_i^2 \norm{P_Au_i}_2^2\\
    &\leq \epsilon^2(n-1)+1\\
    &\leq \epsilon^2(n-1)\left( 1 + \frac{2}{\epsilon^2 n} \right)
\end{split}
\end{align}
 almost surely. Here we assume $n \geq 2$ for the last inequality.
 
 Note that $A = YC^T = \sum_{i=1}^n \sigma_i {u_i} R_i$ where $R_i = \langle v_i, C^T \rangle$ (recall $r=1$ and therefore $C^T$ is $m \times 1$ ). Since $v_i$s are othonormal and entries of $C$ are i.i.d. $\mathcal{N}(0, 1)$, $R_i$s are $\mathcal{N}(0, 1)$ random variables and are independent for all $i$. Since $A$ and $Z_{FA}$ are $n \times 1$ matrices, $B$ and $W_{FA}$ are $1 \times m$ matrices, and $Z_{FA}W_{FA} = AB$ almost surely, we know that $Z_{FA} = cA$ and $W_{FA} = \frac{1}{c}B$ almost surely, for some non-zero constant $c$. Therefore $\frac{Z_{FA}}{\norm{Z_{FA}}_2} = \frac{A}{\norm{A}_2}$ almost surely. Therefore, we get
 \begin{align}
 \label{eq:thm_rep_orth_3}
 \begin{split}
      \left\langle \frac{Z_{FA}}{\norm{Z_{FA}}_2}, \frac{Z_{GD}}{\norm{Z_{GD}}_2} \right\rangle^2  &=   \left\langle \frac{A}{\norm{A}_2}, u_1 \right\rangle^2 \\
     &=  \frac{\left\langle \sum_{i=1}^n \sigma_i {u_i} R_i, u_1 \right\rangle^2}{\sum_{i=1^n} \sigma_i^2 R_i^2}\\
     &= \frac{R_1^2}{R_1^2 + \epsilon^2 \sum_{i=2}^n R_i^2}\\
     &\leq \frac{R_1^2}{\epsilon^2 \sum_{i=2}^n R_i^2}\\
     &\leq \frac{10}{\epsilon^2 n}
\end{split}
 \end{align}
with probability at least 0.99. For the last inequality, we used tail bounds for normal random variable and chi-squared random variable (Lemma \ref{lem:concentration-chisq}), and a union bound, which give $R_1^2 \leq 9$ and $\sum_{i=2}^n R_i^2 \geq n-1 - 8\sqrt{n} \geq 0.9 n$ (assume $n \geq 10000$) with probability at least 0.99.
From Equations \ref{eq:thm_rep_orth_1}, \ref{eq:thm_rep_orth_2} and \ref{eq:thm_rep_orth_3}, we get
\begin{align*}
    \norm{Z_{FA}W_{FA}-Y}_F^2  &\leq \norm{Z_{GD}W_{GD}-Y}_F^2  \left( 1 + \frac{2}{\epsilon^2n} \right)
\end{align*}
and 
\begin{align*}
      \left| \left\langle \frac{Z_{FA}}{\norm{Z_{FA}}_2}, \frac{Z_{GD}}{\norm{Z_{GD}}_2} \right\rangle \right| 
     &\leq \frac{4}{\epsilon\sqrt{n}}
 \end{align*}
 with probability at least 0.99 (assuming $n \geq 10000$).
\end{proof}

\section{PROOF IDEA OF \citet{baldi2018learning}}
\label{sec:baldi_proof_idea}

\citet[Theorem 8]{baldi2018learning} show convergence of feedback alignment (FA) for training two layer linear neural networks. Here, we discuss the main idea behind their proof. We describe the idea for the problem of matrix factorization. Recall that the FA update is given by
\begin{align}
\begin{split}
    \frac{dZ}{dt} &= (Y - \hat{Y})C^T\\
    \frac{dW}{dt} &= Z^T(Y - \hat{Y})
\end{split}
\end{align}
We want to show that $\norm{(Y - \hat{Y})C^T}_F^2$ converges to $0$ with time. 

Let 
\begin{align*}
    V = \frac{1}{2}\left(CW^TWC^T - CY^TZ - Z^TYC^T\right)
\end{align*}
Observe that 
\begin{align*}
    \frac{dV}{dt} = -C(Y - \hat{Y})^T(Y - \hat{Y})C^T
\end{align*}
which implies
\begin{align*}
    \frac{d ~ Tr(V)}{dt} = -  \norm{(Y - \hat{Y})C^T}_F^2 \leq 0.
\end{align*}
Therefore, $Tr(V)$ is monotonically non-increasing with its rate of decrease given by $\norm{(Y - \hat{Y})C^T}_F^2$.

Now, since $CW^TWC^T$ is PSD, $Tr(CW^TWC^T) \geq 0$. Also, $Tr(CY^TZ) = Tr(Z^TYC^T)$. This gives
\begin{align*}
    Tr(V) \geq - Tr\left( CY^TZ \right)
\end{align*}

\citet{baldi2018learning} show that $Z$ remains bounded throughout the dynamics. Therefore, $Tr(V)$ is bounded from below. 

Since $Tr(V)$ is monotonically non-increasing with rate of decrease given by $\norm{(Y - \hat{Y})C^T}_F^2$, and  it is also bounded from below, $\norm{(Y - \hat{Y})C^T}_F^2$ can not be too large for too long. \citet{baldi2018learning} use this observation to show that  $\norm{(Y - \hat{Y})C^T}_F^2$ converges to $0$.

Note that this proof of convergence for FA is very different from our proof of convergence for FA*. This proof doesn't say much about the dynamics of alignment. In contrast, our proof  crucially relies on the phenomenon of alignment and sheds light on how it facilitates convergence.

\section{SIMULATION DETAILS AND ADDITIONAL PLOTS}
\label{sec:sim_details}
\paragraph{Figure \ref{fig:separation_GD_FA}.} We generate $Y$ as $U \Sigma V^T$ where $U$ and $V$ are an $n \times k$ and $m \times k$ independent random matrices respectively with orthonormal columns, and $\Sigma$ is a $k \times k$ diagonal matrix. For GD and FA, $Z$ and $W$ are initialized as $n \times r$ and $r \times m$ random matrices respectively where each entry is drawn i.i.d. from a normal distribution with mean $0$ and standard deviation $0.001$. We use the same initial $Z$ and $W$ for FA and GD. For FA*, we use the same initial $Z$ as FA and GD but $W$ is initialized to $(Z^TZ)^{-1}Z^TY$. $C$ is a random $r \times m$ matrix with i.i.d. entries drawn from a normal distribution with mean $0$ and standard deviation $1$. We use the same $C$ for $FA$, $GD$ and $FA*$. We use a learning rate of $1$ for GD and $0.1$ for FA and FA*.  

For Figure \ref{fig:separation_GD_FA_1}, $n = m = 500$ and $r = k = 50$. The diagonal entries of $\Sigma$ are set to $1/\sqrt{50}$. For Figure \ref{fig:separation_GD_FA_2}, $n = m = k = 500$ and $r = 50$. The first $50$ diagonal entries of $\Sigma$ are set to $1/\sqrt{2*50}$ and the next $450$ diagonal entries are set to $1/\sqrt{2*450}$. The diagonal entries for \ref{fig:separation_GD_FA_2} are set in accordance with Theorem \ref{thm:FA_separation_GD}. 

\paragraph{Figure \ref{fig:dynamics_1}.} $y$ is a random 100 dimensional unit vector. $Z$ is initialized as a $100 \times 50$ random matrix with entries drawn i.i.d. from a normal distribution with mean $0$ and standard deviation $0.001$. $w$ is initialized as $(Z^TZ)^{-1}Z^Ty$. $c$ is set to $-w(0)$. We use a learning rate of $0.1$.

\paragraph{Figure \ref{fig:dynamics_2}.} We generate $Y$ as $AB^T/\norm{AB^T}_F$ where $A$ and $B$ are $100 \times 99$ matrices with with entries drawn i.i.d. from a normal distribution with mean $0$ and unit standard deviation. $Z$ is initialized as a $100 \times 99$  random matrix with orthonormal columns (such that $Z^TZ = I$), and $W$ is initialized as $(Z^TZ)^{-1}Z^TY$. $C$ is a random $99 \times 100$ matrix with i.i.d. entries drawn from a normal distribution with mean $0$ and standard deviation $1$. We use a learning rate of $0.1$. We note that the non-monotonic loss progression is not due to any learning rate issue. $\frac{d \ \norm{(Y-\hat{Y})C^T}_F^2}{dt}$ does switch from being negative to positive and back many times in the dynamics.

\paragraph{Figure \ref{fig:alignment}} This corresponds to the same setting as Figure \ref{fig:dynamics_2}. We draw $x$ from uniform distribution over the unit sphere.

\paragraph{Figure \ref{fig:dynamics_comparison}} This corresponds to the  same setting as Figure \ref{fig:dynamics_2}. We use the same $Z(0), W(0), Y$ and $C$  for FA and FA*. For FA and FA*, we use a learnig rate of $0.1$. For FA with larger learning rate for $W$, we use a learning rate of $0.5$ for $W$ update and $0.1$ for $Z$ update. \vspace{10pt}

\begin{figure*}[h!]
    \centering 
\begin{subfigure}{0.5\textwidth}
  \includegraphics[width=  \linewidth]{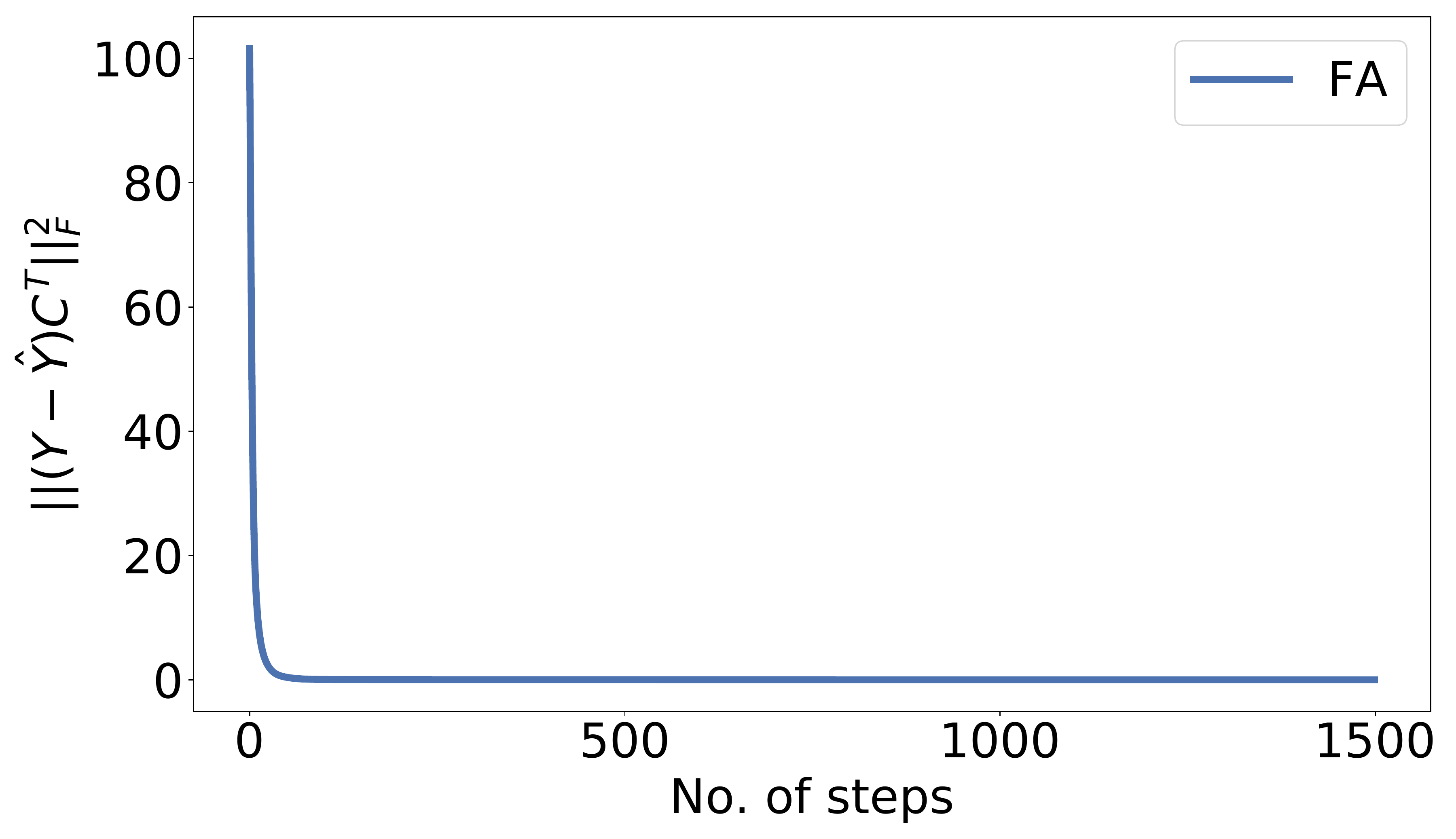}
  \caption{}
  \label{fig:fa_rand_init1}
\end{subfigure}
\begin{subfigure}{0.5\textwidth}
  \includegraphics[width=\linewidth]{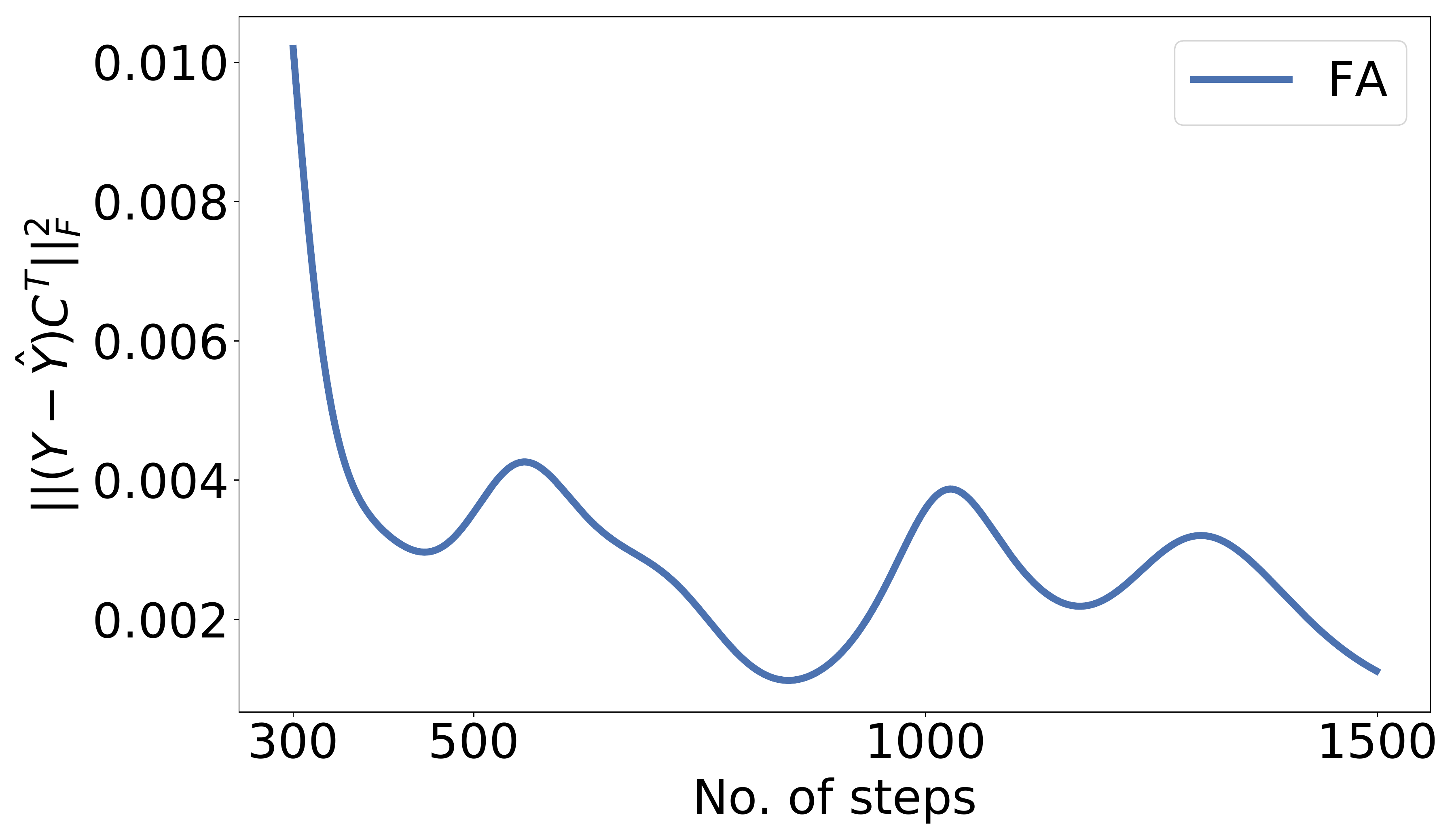}
  \caption{}
  \label{fig:fa_rand_init2}
\end{subfigure}

\caption{(a) FA loss dynamics with the same $Y, C$ and $Z(0)$ as in Figure \ref{fig:dynamics_comparison} but with a randomly initialized $W$. As expected, FA loss progression in this case looks very different from the loss progression for FA* and for FA with optimally initialized $W$, as shown in Figure \ref{fig:dynamics_comparison}.  (b) Zoomed in version of Figure \ref{fig:fa_rand_init1} starting from step $300$ showing non-monotonic loss progression.  }
\label{fig:fa_rand_init}
\end{figure*}

\paragraph{Figure \ref{fig:fa_rand_init}} 
Similar to Figure \ref{fig:dynamics_comparison}, this corresponds to the  same setting as Figure \ref{fig:dynamics_2}, except that $W$ is initialized to a random $99 \times 100$ matrix with i.i.d. entries drawn from a normal distribution with mean $0$ and standard deviation $0.001$.

As expected, the loss progression for FA with randomly initialized $W$ is very different from FA* and from FA with optimally initialized $W$ (shown in Figure \ref{fig:dynamics_comparison}). The initial loss is much higher in this case compared to Figure \ref{fig:dynamics_comparison} as we do not initialize $W$ optimally here.

\end{document}